\documentclass{article} 
\usepackage{iclr2022_conference,times}
\iclrfinalcopy

\usepackage{amsmath,amsfonts,bm}

\def\eqref#1{equation~\ref{#1}}
\def\Eqref#1{Equation~\ref{#1}}

\def\1{\bm{1}}

\DeclareMathAlphabet{\mathsfit}{\encodingdefault}{\sfdefault}{m}{sl}
\SetMathAlphabet{\mathsfit}{bold}{\encodingdefault}{\sfdefault}{bx}{n}

\usepackage{hyperref}
\usepackage{url}
\usepackage{amsthm}
\usepackage{threeparttable}
\usepackage{booktabs}
\usepackage{graphicx}
\usepackage{subfigure}
\usepackage{wrapfig}

\usepackage{color}
\usepackage{amssymb}
\usepackage{bm}
\usepackage{subfigure}
\usepackage{makecell}
\usepackage{threeparttable}
\usepackage{multirow}
\usepackage{caption}
\usepackage{algorithm}
\usepackage{algorithmic}
\usepackage{mathrsfs}

\newtheorem{theorem}{Theorem}
\newtheorem{lemma}{Lemma}

\title{Optimal ANN-SNN Conversion for High-accuracy and Ultra-low-latency Spiking Neural Networks}

\author{Tong Bu\textsuperscript{\rm 1}, Wei Fang\textsuperscript{\rm 1}, Jianhao Ding\textsuperscript{\rm 1}, PengLin Dai\textsuperscript{\rm 2}, Zhaofei Yu\textsuperscript{\rm 1 \rm *}, Tiejun Huang\textsuperscript{\rm 1}\\
\textsuperscript{\rm 1} Peking University,
\textsuperscript{\rm 2} Southwest Jiaotong University\\
\textsuperscript{\rm *} Corresponding author: yuzf12@pku.edu.cn 
}

\begin{document}

\maketitle
\begin{abstract}
Spiking Neural Networks (SNNs) have gained great attraction due to their distinctive properties of low power consumption and fast inference on neuromorphic hardware. As the most effective method to get deep SNNs, ANN-SNN conversion has achieved comparable performance as ANNs on large-scale datasets. Despite this, it requires long time-steps to match the firing rates of SNNs to the activation of ANNs. As a result, the converted SNN suffers severe performance degradation problems with short time-steps, which hamper the practical application of SNNs. In this paper, we theoretically analyze ANN-SNN conversion error and derive the estimated activation function of SNNs. Then we propose the quantization clip-floor-shift activation function to replace the ReLU activation function in source ANNs, which can better approximate the activation function of SNNs. We prove that the expected conversion error between SNNs and ANNs is zero, enabling us to achieve high-accuracy and ultra-low-latency SNNs. We evaluate our method on CIFAR-10/100 and ImageNet datasets, and show that it outperforms the state-of-the-art ANN-SNN and directly trained SNNs in both accuracy and time-steps. To the best of our knowledge, this is the first time to explore high-performance ANN-SNN conversion with ultra-low latency (4 time-steps). Code is available at https://github.com/putshua/SNN\_conversion\_QCFS

\end{abstract}

\section{Introduction}
Spiking neural networks (SNNs) are biologically plausible neural networks based on the dynamic characteristic of biological neurons~\citep{mcculloch1943logical, izhikevich2003simple}. As the third generation of artificial neural networks \citep{maas1997networks}, SNNs have attracted great attention due to their distinctive properties over deep analog neural networks (ANNs)~\citep{roy2019towards}. Each neuron transmits discrete spikes to convey information when exceeding a threshold. 
For most SNNs, the spiking neurons will accumulate the current of the last layer as the output within $T$ inference time steps. The binarized activation has rendered dedicated hardware of neuromorphic computing~\citep{pei2019towards,debole2019truenorth, davies2018loihi}. This kind of hardware has excellent advantages in temporal resolution and energy budget. Existing work has shown the potential of tremendous energy saving with considerably fast inference~\citep{stockl2021optimized}.

In addition to efficiency advantages, the learning algorithm of SNNs has been improved by leaps and bounds in recent years. The performance of SNNs trained by backpropagation through time and ANN-SNN conversion techniques has gradually been comparable to ANNs on large-scale datasets~\citep{fang2021deep, rueckauer2017conversion}. Both techniques benefit from the setting of SNN inference time. Setting longer time-steps in backpropagation can make the gradient of surrogate functions more reliable~\citep{wu2018STBP, neftci2019surrogate,zenke2021remarkable}. However, the price is enormous resource consumption during training. Existing platforms such as TensorFlow and PyTorch based on CUDA have limited optimization for SNN training. In contrast, ANN-SNN conversion usually depends on a longer inference time to get comparable accuracy as the original ANN~\citep{sengupta2019going} because it is based on the equivalence of ReLU activation and integrate-and-fire model's firing rate~\citep{cao2015spiking}. Although longer inference time can further reduce the conversion error, it also hampers the practical application of SNNs on neuromorphic chips.

The dilemma of ANN-SNN conversion is that there exists a remaining potential in the conversion theory, which is hard to be eliminated in a few time steps~\citep{rueckauer2016theory}. Although many methods have been proposed to improve the conversion accuracy, such as weight normalization~\citep{diehl2015fast}, threshold rescaling~\citep{sengupta2019going}, soft-reset~\citep{han2020deep} and threshold shift~\citep{deng2020optimal}, tens to hundreds of time-steps in the baseline works are still unbearable. To obtain high-performance SNNs with ultra-low latency (e.g., 4 time-steps), we list the critical errors in ANN-SNN conversion and provide solutions for each error.
Our main contributions are summarized as follows:

\begin{itemize}
    \item We go deeper into the errors in the ANN-SNN conversion and ascribe them to clipping error, quantization error, and unevenness error. We find that unevenness error, which is caused by the changes in the timing of arrival spikes and has been neglected in previous works, can induce more spikes or fewer spikes as expected.
    \item We propose the quantization clip-floor-shift activation function to replace the ReLU activation function in source ANNs, which better approximates the activation function of SNNs. We prove that the expected conversion error between SNNs and ANNs is zero, indicating that we can achieve high-performance converted SNN at ultra-low time-steps.
    \item We evaluate our method on CIFAR-10, CIFAR-100, and ImageNet datasets. Compared with both ANN-SNN conversion and backpropagation training methods, the proposed method exceeds state-of-the-art accuracy with fewer time-steps. For example, we reach top-1 accuracy 91.18\% on CIFAR-10 with unprecedented 2 time-steps.
\end{itemize}

\section{preliminaries}
In this section, we first briefly review the neuron models for SNNs and ANNs. Then we introduce the basic framework for ANN-SNN conversion. 

\textbf{Neuron model for ANNs.} For ANNs, the computations of analog neurons can be simplified as the combination of a linear transformation and a non-linear mapping:
\begin{align}
    \label{ann}
	\bm{a}^{l} = h(\bm{W}^{l}\bm{a}^{l-1}),~~~l=1,2,...,M
\end{align}
where the vector $\bm{a}^l$ denotes the output of all neurons in $l$-th layer, $\bm{W}^{l}$ denotes the weight matrix between layer $l$ and layer $l-1$, and $h(\cdot)$ is the ReLU activation function.

\textbf{Neuron model for SNNs.} Similar to the previous works~\citep{cao2015spiking,diehl2015fast,han2020rmp}, we consider the Integrate-and-Fire (IF) model for SNNs. If the IF neurons in $l$-th layer receive the input $\bm{x}^{l-1}(t)$ from last layer,
the temporal potential of the IF neurons can be defined as:
\begin{align}
\bm{m}^l(t)&=\bm{v}^l(t-1)+\bm{W}^{l}\bm{x}^{l-1}(t),  \label{firem} 
\end{align}
where $\bm{m}^l(t)$ and $\bm{v}^{l}(t)$ represent the membrane potential before and after the trigger of a spike at time-step $t$. $\bm{W}^{l}$ denote the weight in $l$-th layer.
As soon as any element $m_i^l(t)$ of $\bm{m}^{l}(t)$ exceeds the firing threshold $\theta^{l}$, the neuron will elicit a spike and update the membrane potential $v_i^l(t)$. To avoid information loss, we use the ``reset-by-subtraction'' mechanism~\citep{rueckauer2017conversion, han2020rmp} instead of the ``reset-to-zero'' mechanism, which means the membrane potential $v_i^l(t)$  is subtracted by the threshold value $\theta^{l}$ if the neuron fires.
Based on the threshold-triggered firing mechanism and the ``reset-by-subtraction'' of the membrane potential after firing discussed above, we can write the uplate rule of membrane potential as:
\begin{align}
\bm{s}^l(t)&=H (\bm{m}^l(t)-\bm{\theta}^l), \label{fires} \\
\bm{v}^{l}(t)&=\bm{m}^l(t)-\bm{s}^l(t)\theta^l.  \label{firev} 
\end{align}
Here $\bm{s}^l(t)$ refers to the output spikes of all neurons in layer $l$ at time $t$, the element of which equals 1 if there is a spike and 0 otherwise.
$H(\cdot)$  is the Heaviside step function. $\bm{\theta}^l$ is the vector of the firing threshold $\theta^{l}$. Similar to~\citet{deng2020optimal}, we
suppose that the postsynaptic neuron in $l$-th layer receives unweighted postsynaptic potential  $\theta^l$ if the presynaptic neuron in ${l-1}$-th layer fires a spike, that is:
\begin{align}
    \bm{x}^{l}(t)&=\bm{s}^l(t)\theta^l. \label{fire2}
\end{align}

\textbf{ANN-SNN conversion.} The key idea of ANN-SNN conversion is to map the activation value of an analog neuron in ANN to the firing rate (or average postsynaptic potential) of a spiking neuron in SNN. Specifically, we can get the potential update equation by combining \Eqref{firem} -- \Eqref{firev}:
\begin{align}
    \label{neuron}
    \bm{v}^{l}(t)-\bm{v}^{l}(t-1)=\bm{W}^{l}\bm{x}^{l-1}(t)-\bm{s}^l(t)\theta^l.
\end{align}
\Eqref{neuron} describes the basic function of spiking neurons used in ANN-SNN conversion. 
By summing \Eqref{neuron} from time $1$ to $T$  and dividing $T$ on both sides, we have:
\begin{align}
& \frac{\bm{v}^{l}(T)-\bm{v}^{l}(0)}{T}=\frac{\bm{W}^{l}\sum_{i=1}^{T} \bm{x}^{l-1}(i)}{T} - \frac{\sum_{i=1}^{T}\bm{s}^l(i)\theta^{l}}{T} \label{neuronT}.
\end{align}
If we use $ \bm{\phi}^{l-1}(T) =\frac{\sum_{i=1}^{T} \bm{x}^{l-1}(i)}{T}$ to denote the average postsynaptic potential during the period from 0 to $T$ and substitute \Eqref{fire2} into \Eqref{neuronT}, then we get:
\begin{align}
    \bm{\phi}^l(T)=\bm{W}^{l} \bm{\phi}^{l-1}(T) -\frac{\bm{v}^{l}(T)-\bm{v}^{l}(0)}{T}.
    \label{postpoten}
\end{align}
\Eqref{postpoten} describes the  relationship of the average postsynaptic potential of neurons in adjacent layers. Note that $\bm{\phi}^l(T) \geqslant 0$.
If we set the initial potential $\bm{v}^{l}(0)$  to zero and neglect the remaining term $\frac{\bm{v}^{l}(T)}{T}$ when the simulation time-steps $T$ is long enough, the converted SNN has nearly the same activation function as source ANN (\Eqref{ann}). However, high $T$ would cause long inference latency that hampers the practical application of SNNs. Therefore, this paper aims to implement high-performance ANN-SNN conversion with extremely low latency.

\begin{table}[t]
\renewcommand\arraystretch{1.2}
\caption{Summary of notations in this paper}
\label{tab:notations}
\centering
\scalebox{0.8}
{
\begin{threeparttable}
\begin{tabular}{ll|ll} 
\hline
 \textbf{Symbol}  & \textbf{Definition}     & \textbf{Symbol}   & \textbf{Definition}\\ \hline
 $l$              & Layer index             & $\bm{x}^l(t)$     & Unweighted PSP\tnote{1}\\ 
 $i$              & Neuron index            & $\bm{s}^l(t)$     & Output spikes\\ 
 $\bm{W}^l$       & Weight                  & $\bm{\phi}^l(T)$  & Average unweigthed PSP before time $T$ \\
 $\bm{a}^l$       & ANN activation values   & $\bm{z}^l$        & Weighted input from $l-1$ layer \\
 $t$              & Time-steps              & $h(\cdot)$        & ReLU function  \\
 $T$              & Total time-step         & $H(\cdot)$        & Heaviside step function  \\
 $\bm{\theta}^l$  & Threshold               & $L$               & Quantization step for ANN \\
 $\lambda^l$      & Trainable threshold in ANN & $\bm{Err}^l$      & Conversion Error  \\
 $\bm{m}^l(t)$    & Potential before firing & $\widetilde{\bm{Err}}^l$      & Estimated conversion Error  \\ 
 $\bm{v}^l(t)$    & Potential after firing  & $\varphi$         & Shift of quantization clip-floor function  \\ \hline
\end{tabular}
    \begin{tablenotes}
		\footnotesize
		\item[1] Postsynaptic potential 
	\end{tablenotes}
\end{threeparttable}
}
\end{table}



\section{conversion error analysis}
\label{sec:conversion_error}
In this section, we will analyze the conversion error between the source ANN and the converted SNN in each layer in detail. In the following, we assume that both ANN and SNN receive the same input from the layer $l-1$, that is, $\bm{a}^{l-1}= \bm{\phi}^{l-1}(T)$, and then analyze the error in layer $l$.
For simplicity, we use $\bm{z}^{l} = \bm{W}^{l}\bm{\phi}^{l-1}(T) = \bm{W}^{l}\bm{a}^{l-1}$ to substitute the weighted input from layer $l-1$ for both ANN and SNN. The absolute conversion error is exactly the outputs from converted SNN subtract the outputs from ANN:
\begin{align}
    \label{error1}
	\bm{Err}^{l} = \bm{\phi}^l(T)-	\bm{a}^{l} = \bm{z}^{l}-\frac{\bm{v}^{l}(T)-\bm{v}^{l}(0)}{T} - h(\bm{z}^{l}),
\end{align}
where $ h(\bm{z}^{l})=\text{ReLU} (\bm{z}^{l})$. It can be found from \Eqref{error1} that the conversion error is nonzero if $\bm{v}^{l}(T)-\bm{v}^{l}(0) \ne 0$ and $\bm{z}^l>0$. 
In fact, the conversion error is caused by three factors.


\textbf{Clipping error.} The output $ \bm{\phi}^l(T)$ of SNNs is in the range of $[0, \theta^l ]$ as $ \bm{\phi}^l(T) =\frac{\sum_{i=1}^{T} \bm{x}^{l}(i)}{T}=\frac{\sum_{i=1}^{T} \bm{s}^l(i)}{T}\theta^l$ (see \Eqref{fire2}). However, the output $\bm{a}^{l}$ of ANNs is in a much lager range of $[0, a^l_{max} ]$, where $a^l_{max}$ denotes the  maximum value of $\bm{a}^{l}$. As illustrated in Figure~\ref{fig1}a, $\bm{a}^{l}$ can be mapped to $ \bm{\phi}^l(T)$ by the following equation:
\begin{align}
    \label{mapping}
    \bm{\phi}^l(T)=\text{clip} \left(   \frac{ \theta^l}{T} \left \lfloor \frac{\bm{a}^{l} T}{\lambda^l}     \right \rfloor, 0, \theta^l  \right).
\end{align}
Here the clip function sets the upper bound $\theta^l$ and the lower bound $0$. $\lfloor \cdot \rfloor$ denotes the floor function. $\lambda^l$ represents the actual maximum value of output $\bm{a}^{l}$ mapped to the maximum value $\theta^l$ of $ \bm{\phi}^l(T)$. Considering that nearly 99.9\% activations of $\bm{a}^{l}$ in ANN are in the range of $[0, \frac{a^l_{max}}{3}]$, 
\citet{rueckauer2016theory} suggested to choose $\lambda^l$ according to 99.9\% activations. The activations between $\lambda^l$ and $a^l_{max}$ in ANN are mapped to the same value $\theta^l$ in SNN, which will cause conversion error called clipping error.

\begin{figure}[t] 
\centering
\subfigure[Clipping error]{
\includegraphics[width=0.23\textwidth]{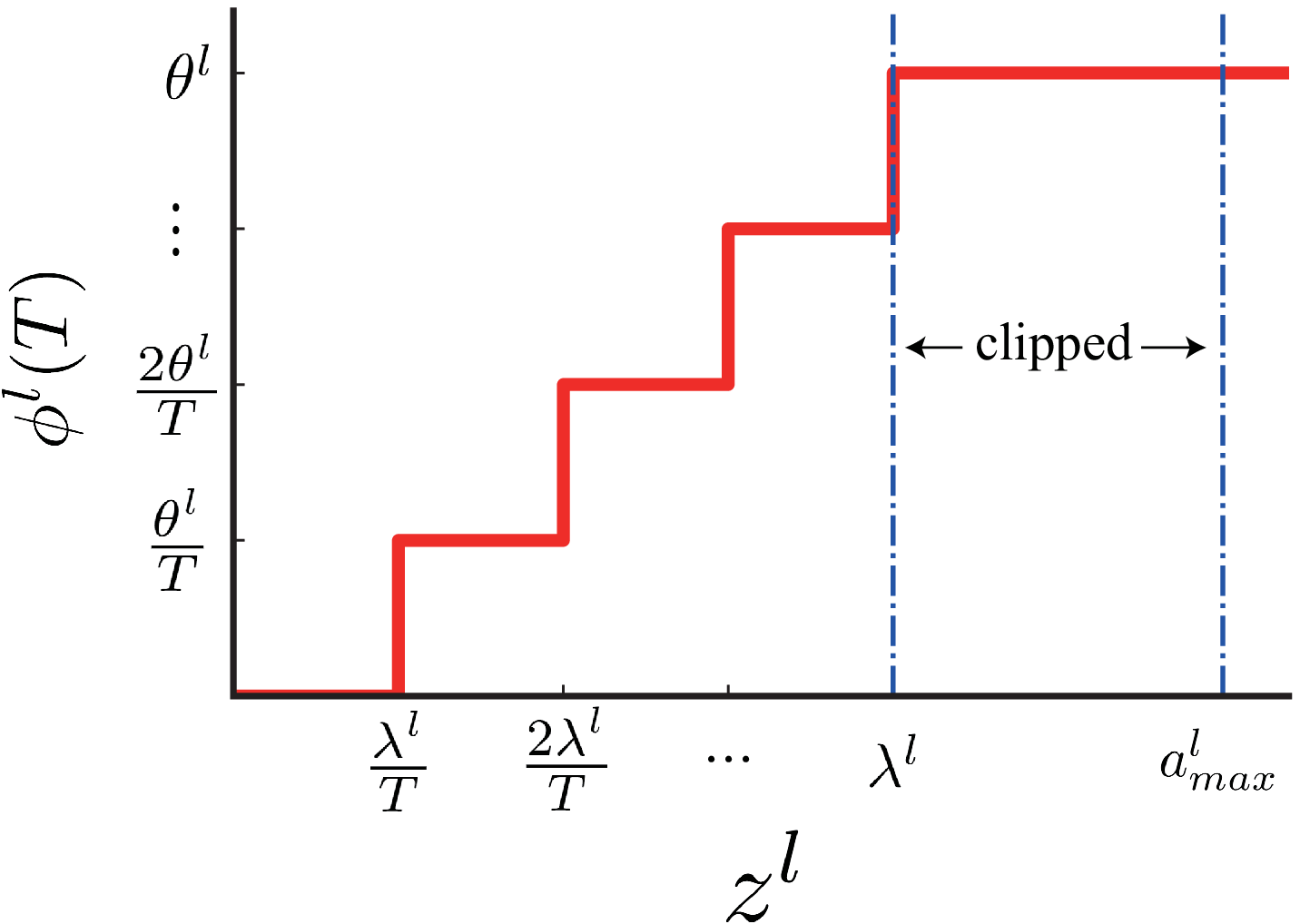}}
\subfigure[Even spikes]{
\includegraphics[width=0.23\textwidth]{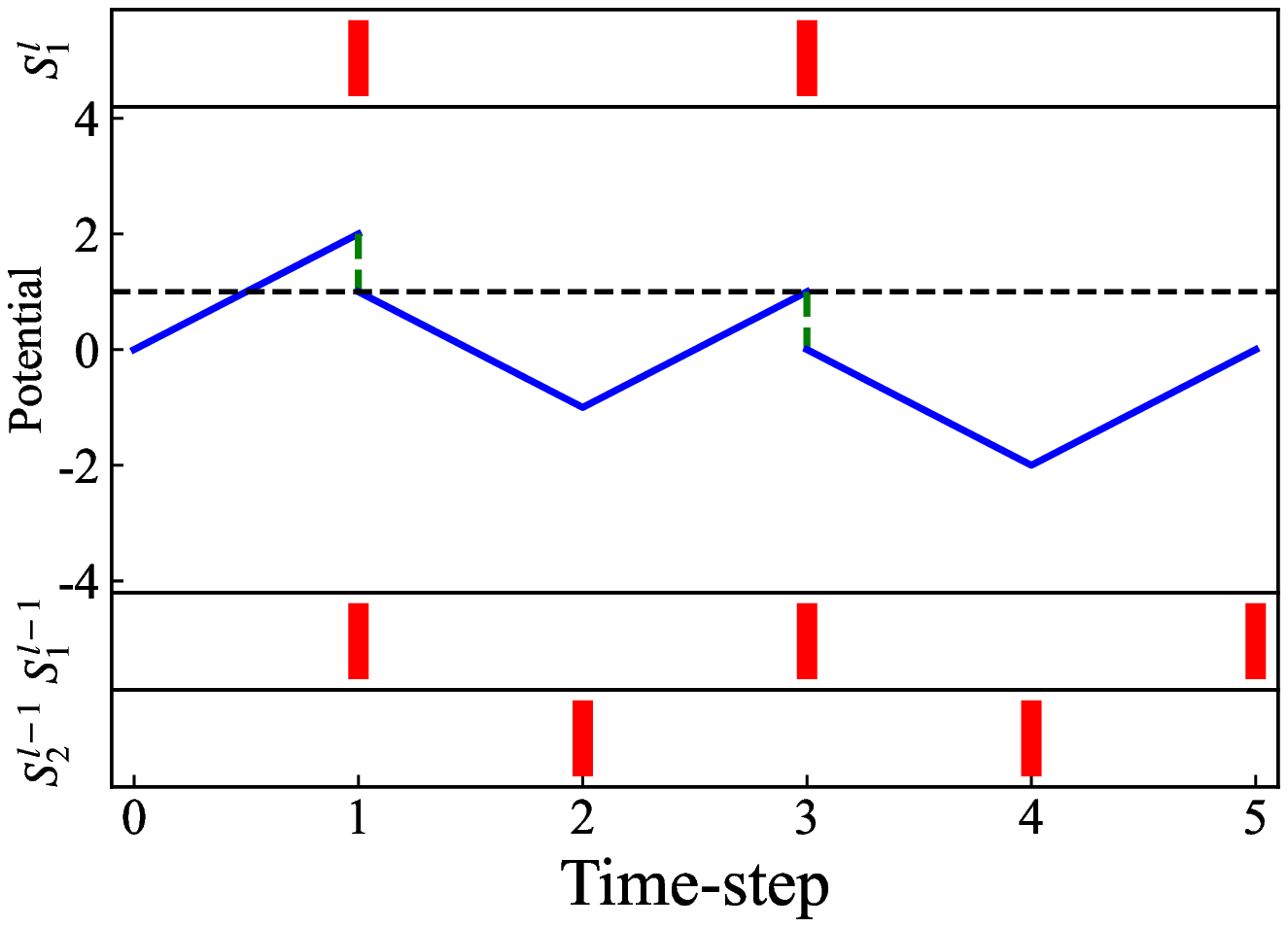}}
\subfigure[More spikes]{
\includegraphics[width=0.23\textwidth]{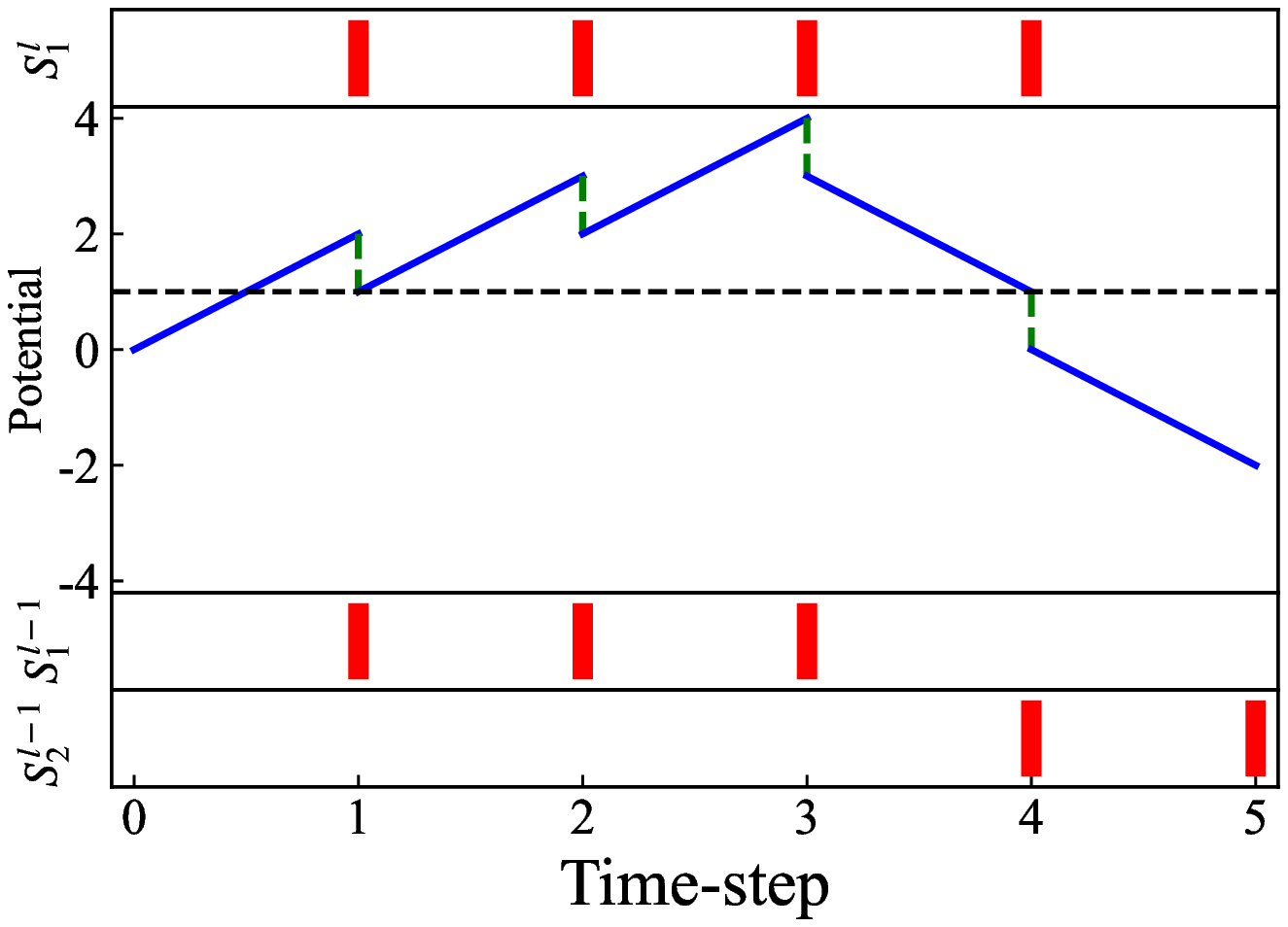}}
\subfigure[Fewer spikes]{
\includegraphics[width=0.23\textwidth]{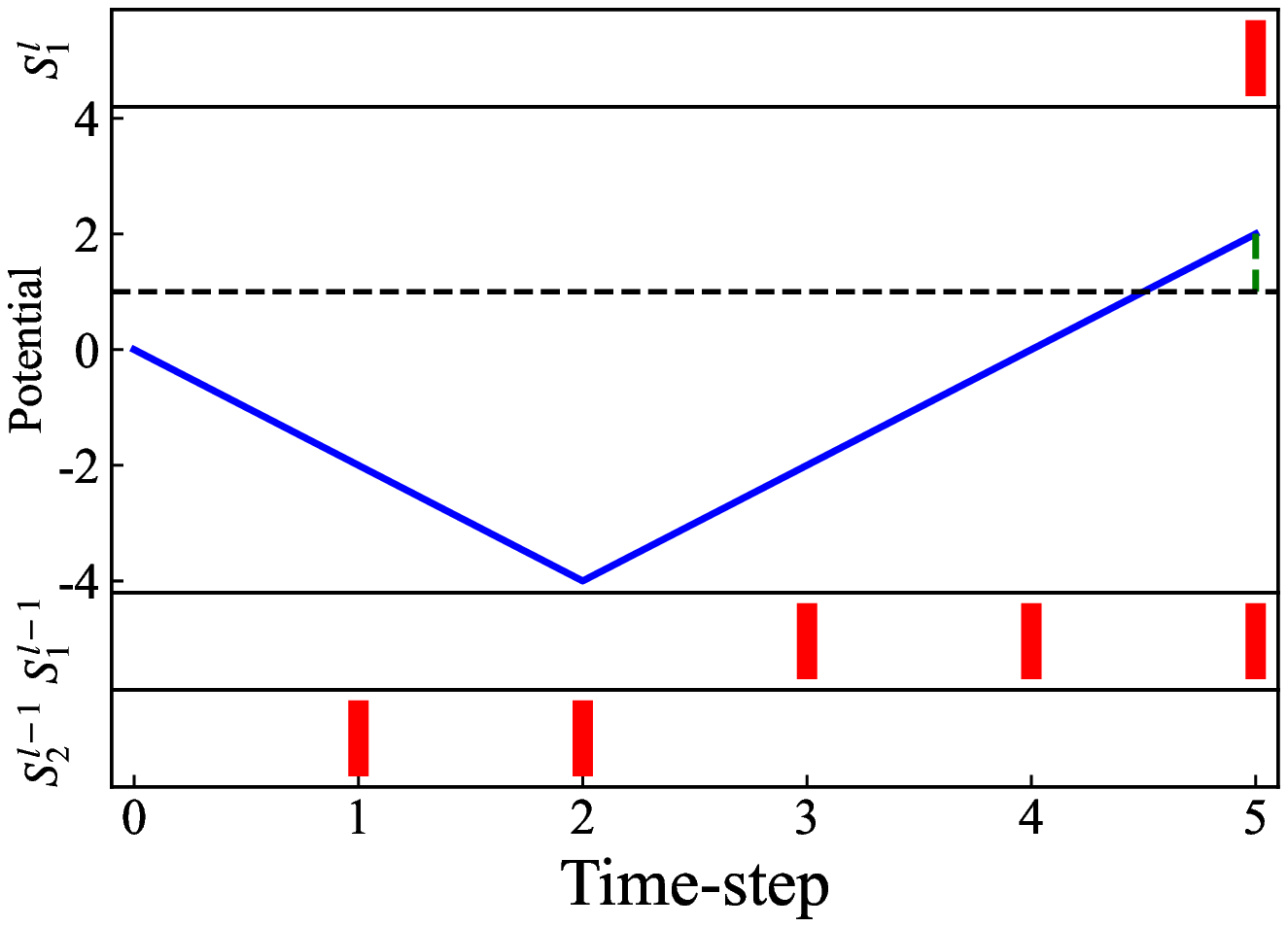}}
\caption{Conversion error between source ANN and converted SNN. $s^{l-1}_1$ and  $s^{l-1}_2$ denote the output spikes of two neurons in layer $l-1$, and $s^{l}_1$ denotes the output spikes of a neuron in layer $l$.}  
\label{fig1} 
\end{figure}

\textbf{Quantization error (flooring error).} The output spikes $\bm{s}^l(t)$ are discrete events, thus $\bm{\phi}^l(T)$ are discrete with quantization resolution $ \frac{ \theta^l}{T}$ (see \Eqref{mapping}). When mapping $\bm{a}^{l}$ to $\bm{\phi}^l(T)$, there exists unavoidable quantization error. For example, as illustrated in Figure~\ref{fig1}a, the activations of ANN in the range of $[\frac{\lambda^l}{T}, \frac{2\lambda^l}{T})$ are mapped to the same value $\frac{\theta^l}{T}$ of SNN.

\textbf{Unevenness error.} Unevenness error is caused by the unevenness of input spikes. If the timing of arrival spikes changes, the output firing rates may change, which causes conversion error. There are two situations: more spikes as expected or fewer spikes as expected. To see this, in source ANN, we suppose that two analog neurons in layer $l-1$ are connected to an analog neuron in layer $l$ with weights 2 and -2, and the output vector $\bm{a}^{l-1}$ of neurons in layer $l-1$ is $[0.6, 0.4]$. Besides, in converted SNN, we suppose that the two spiking neurons in layer $l-1$ fire 3 spikes and 2 spikes in 5 time-steps (T=5), respectively, and the threshold $\theta^{l-1}=1$. Thus, $\bm{\phi}^{l-1}(T)=\frac{\sum_{i=1}^{T} \bm{s}^{l-1}(i)}{T}\theta^{l-1}=[0.6, 0.4]$. Even though $\bm{\phi}^{l-1}(T)= \bm{a}^{l-1}$ and the weights are same for the ANN and SNN,  $\bm{\phi}^{l}(T)$ can be different from $\bm{a}^{l}$ if the timing of arrival spikes changes. According to  \Eqref{ann}, the ANN output $\bm{a}^{l}=\bm{W}^{l}\bm{a}^{l-1}=[2,-2] [0.6, 0.4]^{T}=0.4$. As for SNN, supposing that the threshold  $\theta^{l}=1$, there are three possible output firing rates, which are illustrated in Figure~\ref{fig1} (b)-(d).
If the two presynaptic neurons fires at $t=1,3,5$ and $t=2,4$ (red bars) respectively with weights 2 and -2, the postsynaptic neuron will fire two spikes at $t=1,3$ (red bars), and   $\bm{\phi}^{l}(T)=\frac{\sum_{i=1}^{T} \bm{s}^l(i)}{T}\theta^{l}=0.4=\bm{a}^{l}$.
However, if the presynaptic neurons fires at $t=1,2,3$ and $t=4,5$, respectively, the postsynaptic neuron will fire four spikes at $t=1,2,3,4$, and  $\bm{\phi}^{l}(T)=0.8>\bm{a}^{l}$.  If the presynaptic neurons fires at $t=3,4,5$ and $t=1,2$, respectively, the postsynaptic neuron will fire only one spikes at $t=5$, and  $\bm{\phi}^{l}(T)=0.2<\bm{a}^{l}$. 




Note that the clipping error and quantization error have been proposed in \cite{li2021free}. There exist interdependence between the above three kinds of errors. Specifically, the unevenness error will degenerate to the quantization error if $\bm{v}^{l}(T)$ is in the range of $[0, \theta^l]$. 
Assuming that the potential $\bm{v}^{l}(T)$ falls into $[0, \theta^{l}]$ will enable us to estimate the activation function of SNNs ignoring the effect of unevenness error.
Therefore, an estimation of the
output value $\bm{\phi}^l(T)$ in a converted SNN can be formulated with the combination of clip function and floor function, that is:
\begin{align}
\label{snn_est}
\bm{\phi}^l(T)  \approx \theta^{l}~\mathrm{clip} \left( \frac{1}{T}\left \lfloor \frac{\bm{z}^{l} T  + \bm{v}^{l}(0)}{\theta^{l}} \right \rfloor, 0, 1 \right).
\end{align}
The detailed derivation is in the Appendix. With the help of this estimation for the SNN output, the estimated conversion error $\widetilde{\bm{Err}}^{l}$ can be derived from \Eqref{error1}:
\begin{align}
\label{err_errornew}
	\widetilde{\bm{Err}}^{l} = \theta^{l}~\mathrm{clip} \left(\frac{1}{T}\left \lfloor \frac{\bm{z}^{l} T  + \bm{v}^{l}(0)}{\theta^{l}} \right \rfloor, 0, 1 \right) - h(\bm{z}^{l}) \approx {\bm{Err}}^{l} .
\end{align}

\section{Optimal ANN-SNN conversion}
\label{sec:optimal}

\begin{figure}[t] 
\centering
\subfigure[$L=T=4$]{
\includegraphics[width=0.27\textwidth]{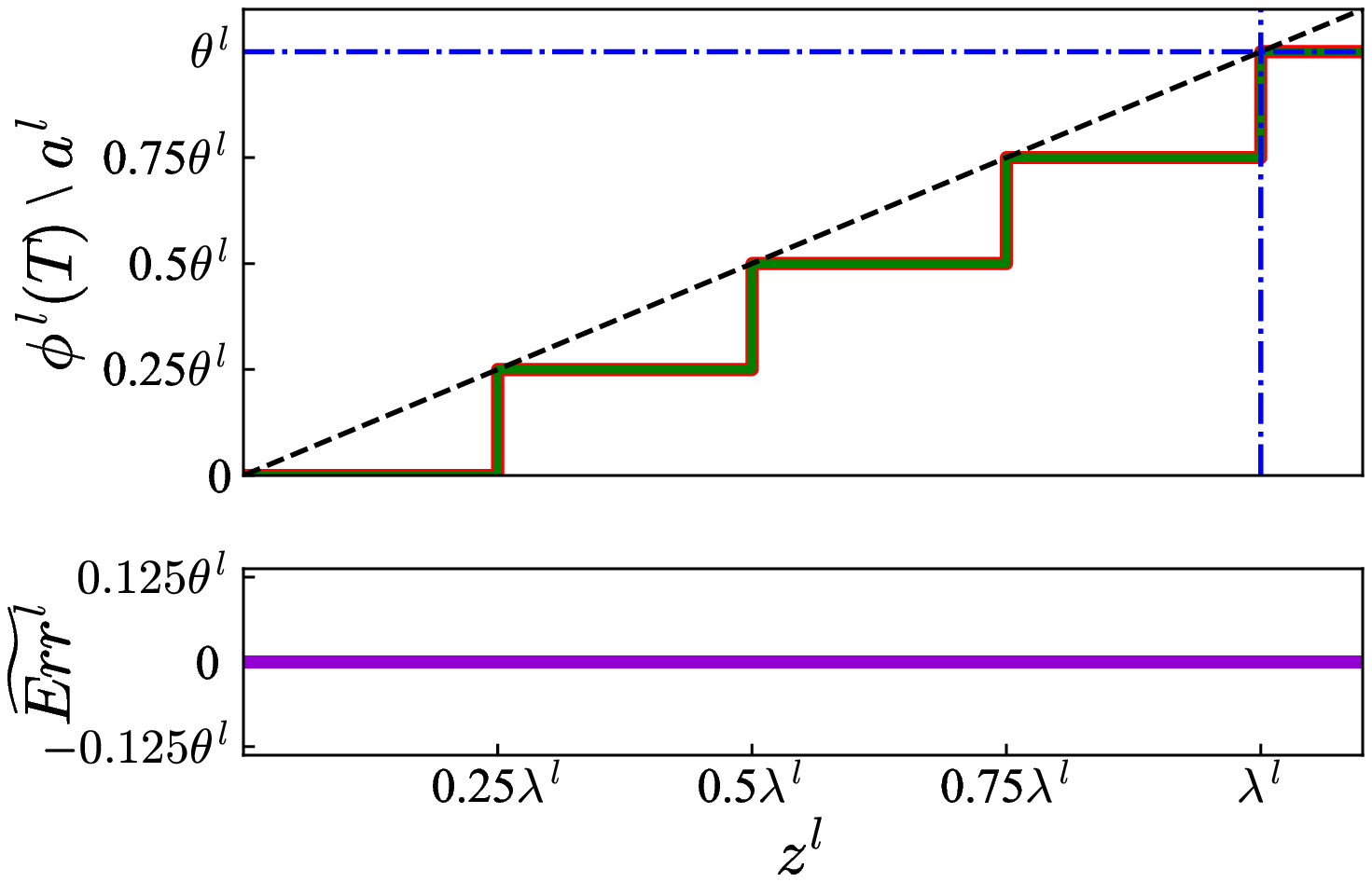}}
\quad\quad
\subfigure[$L=4, T=8$]{
\includegraphics[width=0.27\textwidth]{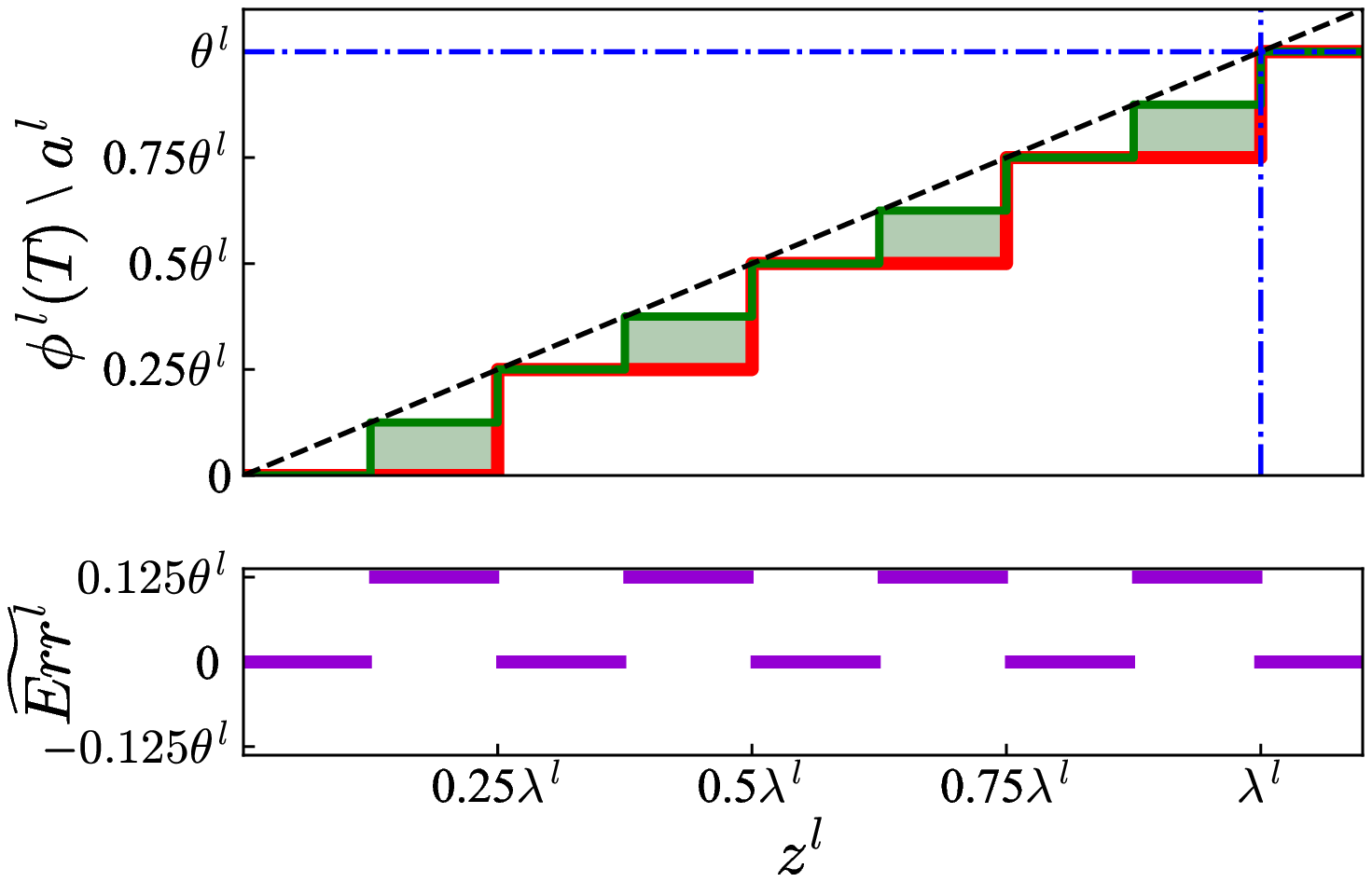}}
\quad\quad
\subfigure[$L=4, T=8, \bm{\varphi}=\bm{0.5}$]{
\includegraphics[width=0.27\textwidth]{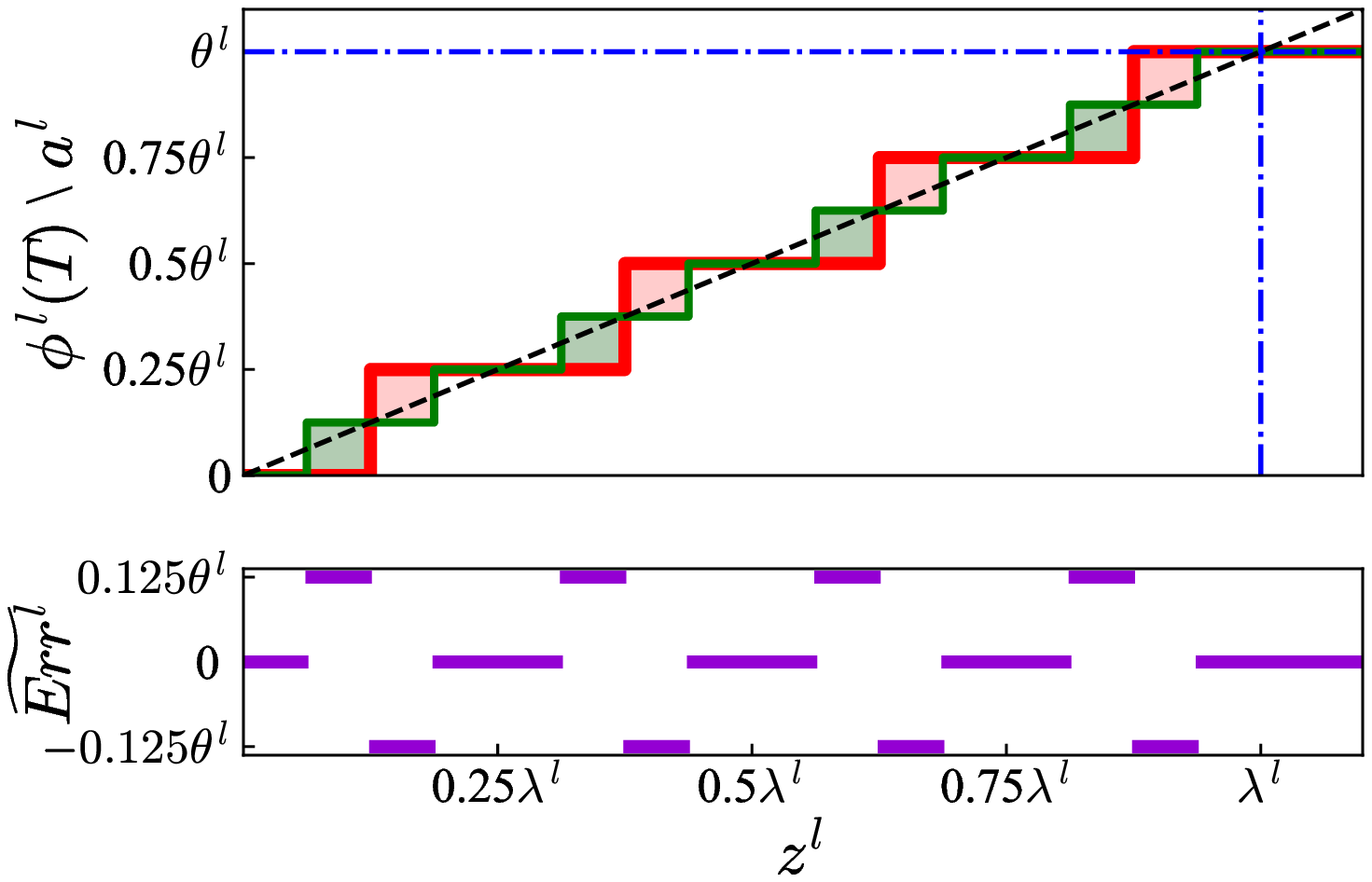}}
\caption{Comparison of SNN output $\bm{\phi}^l(T)$ and ANN output  $\bm{a}^l$ with same input $\bm{z}^l$}
\label{fig2} 
\end{figure}

\subsection{quantization clip-floor activation function}
According to the conversion error of \Eqref{err_errornew}, it is natural to think that if the commonly used ReLU activation function $h(\bm{z}^{l})$ is substituted by a clip-floor function with a given quantization steps $L$ (similar to \Eqref{snn_est}), the conversion error at time-steps $T=L$ will be eliminated. Thus the performance degradation problem at low latency will be solved. 
As shown in \Eqref{annnew}, we proposed the quantization clip-floor activation function to train ANNs.
\begin{align}
    \label{annnew}
	\bm{a}^{l} = \bar{h}(\bm{z}^{l})=\lambda^l~\mathrm{clip} \left(\frac{1}{L}\left \lfloor \frac{\bm{z}^{l}L}{\lambda^l}\right \rfloor, 0, 1 \right),
\end{align}
where the  hyperparameter $L$ denotes quantization steps of ANNs, the trainable $\lambda^l$ decides the maximum value of $\bm{a}^{l}$ in ANNs mapped to the maximum of  $\bm{\phi}^l(T)$ in SNNs.
Note that $\bm{z}^{l} = \bm{W}^{l}\bm{\phi}^{l-1}(T) = \bm{W}^{l}\bm{a}^{l-1}$. With this new activation function, we can prove that the estimated conversion error between SNNs and ANNs is zero, and we have the following Theorem.


 \begin{theorem}
\label{t1}
An ANN with activation function (\ref{annnew}) is converted to an SNN with the same weights. If $T=L$, $\theta^l=\lambda^l$, and $\bm{v}^{l}(0)=\bm{0}$, then:
\begin{align}
    \widetilde{\bm{Err}}^{l} =\bm{\phi}^l(T)-	\bm{a}^{l} =\bm{0}.
\end{align}
\end{theorem}
\begin{proof}
According to \Eqref{err_errornew}, and the conditions $T=L$, $\theta^l=\lambda^l$, $\bm{v}^{l}(0)=\bm{0}$, we have $\widetilde{\bm{Err}}^{l} =\bm{\phi}^l(T)-\bm{a}^{l}=\theta^{l}~\mathrm{clip} \left(\frac{1}{T}\left \lfloor \frac{\bm{z}^{l} T  + \bm{v}^{l}(0)}{\theta^{l}} \right \rfloor, 0, 1 \right) -\lambda^l~\mathrm{clip} \left(\frac{1}{L}\left \lfloor \frac{\bm{z}^{l}L}{\lambda^l}\right \rfloor, 0, 1 \right)=0$. 
\end{proof}
Theorem 1 implies that if the time-steps $T$ of the converted SNN is the same as the quantization steps $L$ of the source ANN, the conversion error will be zero. An example is illustrated in Figure~\ref{fig2}a, where $T=L=4$, $\theta^l=\lambda^l$. The red curve presents the estimated output $\bm{\phi}^l(T)$ of the converted SNNs with respective to different input $\bm{z}^l$, while the green curve represents the out $\bm{a}^l$ of the source ANN with respective to different input $\bm{z}^l$. As the two curve are the same, the estimated conversion error $\widetilde{\bm{Err}}^l$ is zero. 
Nevertheless, in practical application, we focus on the performance of SNNs at different time-steps. There is no guarantee that the conversion error is zero when $T$ is not equal to $L$. As illustrated in Figure~\ref{fig2}b, where $L=4$ and $L=8$, we can find the conversion error is greater than zero for some $\bm{z}^l$. This error will transmit layer-by-layer and eventually degrading the accuracy of the converted SNN. One way to solve this problem is to train multiple source ANNs with different quantization steps, then convert them to SNNs with different time-steps, but it comes at a considerable cost. In the next section, we propose the quantization clip-floor activation function with a shift term to solve this problem. Such an approach can achieve high accuracy for different time-steps, without extra computation cost.

\subsection{quantization clip-floor-shift activation function}
We propose the quantization clip-floor-shift activation function to train ANNs. 
\begin{align}
    \label{annnew2}
	\bm{a}^{l} = \widehat{h}(\bm{z}^{l})=\lambda^l~\mathrm{clip} \left(\frac{1}{L}\left \lfloor \frac{\bm{z}^{l}L}{\lambda^l}+ \bm{\varphi} \right \rfloor, 0, 1 \right).
\end{align}
Compared with \Eqref{annnew}, there exists a hyperparameter vector $\bm{\varphi}$ that controls the shift of the activation function. When $L \ne T$, we cannot guarantee the conversion error is 0. However, we can estimate the expectation of conversion error.
Similar to~\citep{deng2020optimal}, we assume that $z_i^l$ is uniformly distributed within intervals $[(t-1)\lambda^l/T, (t)\lambda^l/T]$ and $[(l-1)\lambda^l/L, (l)\lambda^l/L]$ for $t=1,2,...,T$ and $L=1,2,...,L$, we have the following Theorem.


 \begin{theorem}
\label{t2}
An ANN with activation function (\ref{annnew2}) is converted to an SNN with the same weights. If $\theta^l=\lambda^l$, $\bm{v}^{l}(0)=\theta^l\bm{\varphi}$, 
then for arbitrary $T$ and $L$, the expectation of conversion error reaches $\bm{0}$ when the shift term $\bm{\varphi}$ in source ANN is $\bm{\frac{1}{2}}$.
\begin{align}
\forall\ T,L \quad
\left.\mathbb{E}_z \left ( {\widetilde{\bm{Err}}}^l \right ) \right |_{\bm{\varphi} = \bm{\frac{1}{2}}} &= \bm{0}. 
\end{align}
 \end{theorem}

The proof is in the Appendix. Theorem 2 indicates that the shift term $\bm{\frac{1}{2}}$ is able to optimize the expectation of conversion error. 
By comparing Figure~\ref{fig2}b and  Figure~\ref{fig2}c, we can find that when
the shift term $\bm{\varphi=0.5}$ is added, the mean conversion error reaches zero, even though $L \ne T$. These results indicate we can achieve high-performance converted SNN at ultra-low time-steps.

$L$ is the only undetermined hyperparameter of the quantization clip-floor-shift activation. When $T=L$, the conversion error reaches zero. So we naturally think that the parameter $L$ should be set as small as possible to get better performance at low time-steps. 
However, a too low quantization of the activation function will decrease the model capacity and further lead to accuracy loss when the time-steps is relatively large. 
Choosing the proper $L$ is a trade-off between the accuracy at low latency and the best accuracy of SNNs. We will further analyze the effects of quantization steps $L$ in the experiment section.

\subsection{algorithm for training quantization clip-floor-shift activation function}
Training an ANN with quantization clip-floor-shift activation instead of ReLU is also a tough problem. To direct train the ANN, we use the straight-through estimator~\citep{bengio2013estimating} for the derivative of the floor function, that is $\frac{\mathrm{d}\lfloor x\rfloor}{\mathrm{d}x}=1$. The overall derivation rule is given in \Eqref{train2}. 
\begin{align}
    \frac{\partial \widehat{h}_i(\bm{z}^l)}{\partial z_i^l} &= 
        \begin{cases}
            1,&-\frac{\lambda^l}{2L}< z_i^l <\lambda^l-\frac{\lambda^l}{2L} \\
            0,&\mathrm{otherwise}
        \end{cases},
    \frac{\partial \widehat{h}_i(\bm{z}^l)}{\partial \lambda^l} &=
        \begin{cases}
            \frac{\widehat{h}_i(\bm{z}^l)-z^l_i}{\lambda^l}, &-\frac{\lambda^l}{2L} \leqslant z^l_i <\lambda^l-\frac{\lambda^l}{2L} \\
            0,& z^l_i < -\frac{\lambda^l}{2L}  \\
            1,& z^l_i \geqslant \lambda^l-\frac{\lambda^l}{2L}
        \end{cases}
        \label{train2}
\end{align}
Here $z_i^l$ is the i-th element of $\bm{z}^l$. Then we can train the ANN with quantization clip-floor-shift activation using Stochastic Gradient Descent algorithm~\citep{bottou2012stochastic}. 
\section{Related Work}



The study of ANN-SNN conversion is first launched by~\cite{cao2015spiking}. Then~\cite{diehl2015fast} converted a three-layer CNN to an SNN using data-based and model-based normalization. To obtain high-performance SNNs for complex datasets and deeper networks,~\cite{rueckauer2016theory} and~\cite{sengupta2019going} proposed more accurate scaling methods to normalize weights and scale thresholds respectively, which were later proved to be equivalent~\citep{ding2021optimal}. Nevertheless, the converted deep SNN requires hundreds of time steps to get accurate results due to the conversion error analyzed in Sec.~\ref{sec:conversion_error}. To address the potential information loss,~\cite{rueckauer2016theory} and~\cite{han2020rmp} suggested using ``reset-by-subtraction'' neurons rather than ``reset-to-zero'' neurons. Recently, many methods have been proposed to eliminate the conversion error. \cite{rueckauer2016theory}~recommended 99.9\% percentile of activations as scale factors, and~\cite{ho2020tcl} added the trainable clipping layer. Besides, \cite{han2020rmp} rescaled the SNN thresholds to avoid the improper activation of spiking neurons. \cite{massa2020efficient} and \cite{singh2021gesture} evaluated the performance of converted SNNs on the Loihi Neuromorphic Processor.
Our work share similarity with \cite{deng2020optimal,li2021free}, which also shed light on the conversion error. \cite{deng2020optimal} minimized the layer-wise error by introducing extra bias in addition to the converted SNN biases. \cite{li2021free} further proposed calibration for weights and biases using quantized fine-tuning. They got good results with 16 and 32 time-steps without trails for more extreme time-steps. In comparison, our work aims to fit ANN into SNN with techniques eliminating the mentioned conversion error. The end-to-end training of quantization layers is implemented to get better overall performance. Our shift correction can lead to a single SNN which performs well at both ultra-low and large time-steps.
Maintaining SNN performance within extremely few time-steps is difficult even for supervised learning methods like backpropagation through time (BPTT). BPTT usually requires fewer time-steps because of thorough training, yet at the cost of heavy GPU computation~\citep{wu2018STBP, wu2019direct,lee2016training,neftci2019surrogate,lee2020enabling,zenke2021remarkable}. 
The timing-based backpropagation methods~\citep{bohte2002error,tavanaei2019deep,kim2020unifying} could train SNNs over a very short temporal window, e.g. over 5-10 time-steps. However, they are usually limited to simple datasets like MNIST~\citep{kheradpisheh2020temporal} and CIFAR10~\citep{zhang2020temporal}.
\cite{rathi2019enabling} shortened simulation steps by initializing SNN with conversion method and then tuning SNN with STDP. In this paper, the proposed method achieves high-performance SNNs with ultra-low latency (4 time-steps).

\section{Experiments}
In this section, we validate the effectiveness of our method and compare our method with other state-of-the-art approaches for image classification tasks on CIFAR-10~\citep{lecun1998gradient}, CIFAR-100~\citep{krizhevsky2009learning}, and ImageNet datasets~\citep{deng2009imagenet}.
Similar to previous works, we utilize VGG-16~\citep{simonyan2014very}, ResNet-18~\citep{he2016deep}, and ResNet-20 network structures for source ANNs.  
We compare our method with the state-of-the-art ANN-SNN conversion methods, including Hybrid-Conversion (HC) from \cite{rathi2019enabling}, RMP from \cite{han2020rmp}, TSC from \cite{han2020deep}, RNL from \cite{ding2021optimal}, ReLUThresholdShift (RTS) from \cite{deng2020optimal}, and SNN Conversion with Advanced Pipeline  (SNNC-AP) from \cite{li2021free}. Comparison with different SNN training methods is also included to manifest the superiority of low latency inference, including HybridConversion-STDB (HC-STDB) from \cite{rathi2019enabling}, STBP from \cite{wu2018STBP}, DirectTraining  (DT) from \cite{wu2019direct}, and TSSL from \cite{zhang2020temporal}. 
The details of the proposed ANN-SNN algorithm and training configurations are provided in the Appendix.


\subsection{Test accuracy of ANN with quantization clip-floor-shift activation}


\begin{figure}[t] 
\centering
\includegraphics[width=0.23\textwidth]{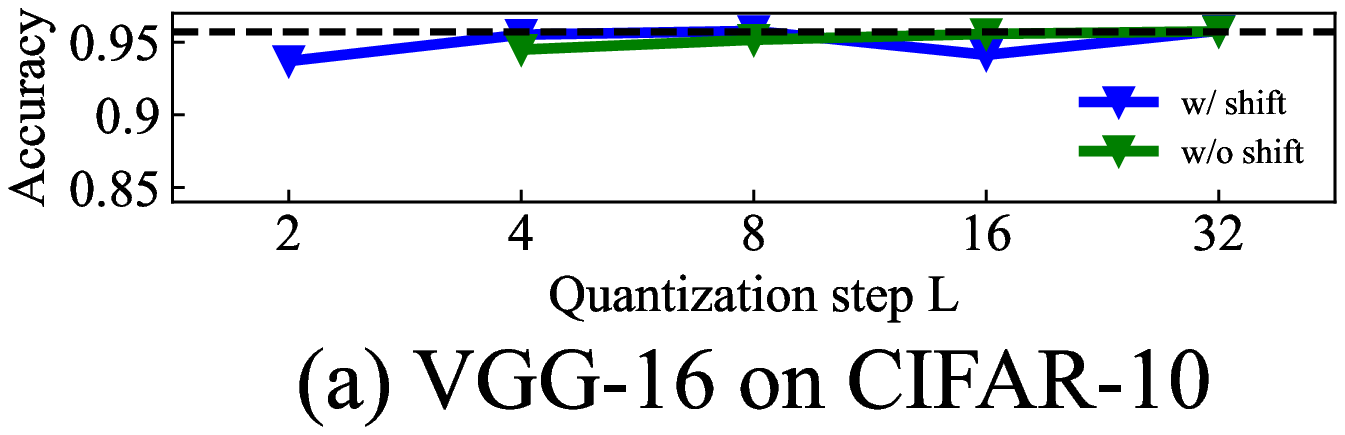}
\includegraphics[width=0.23\textwidth]{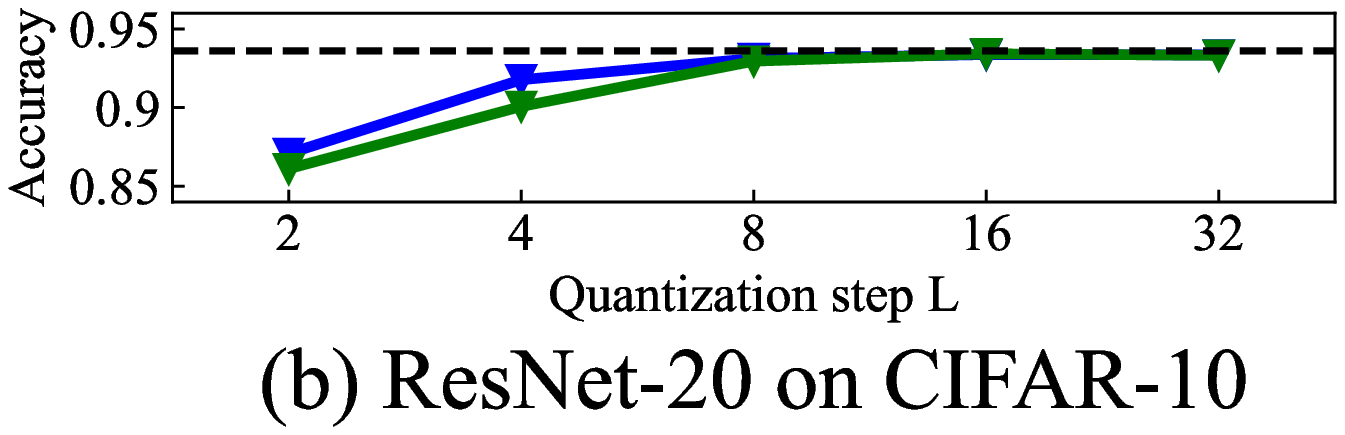}
\includegraphics[width=0.23\textwidth]{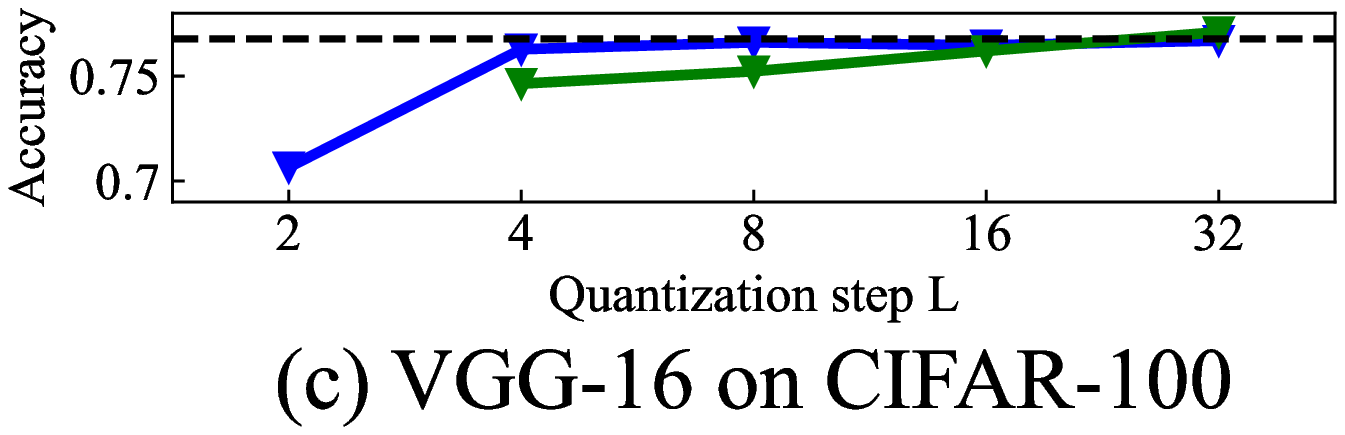}
\includegraphics[width=0.23\textwidth]{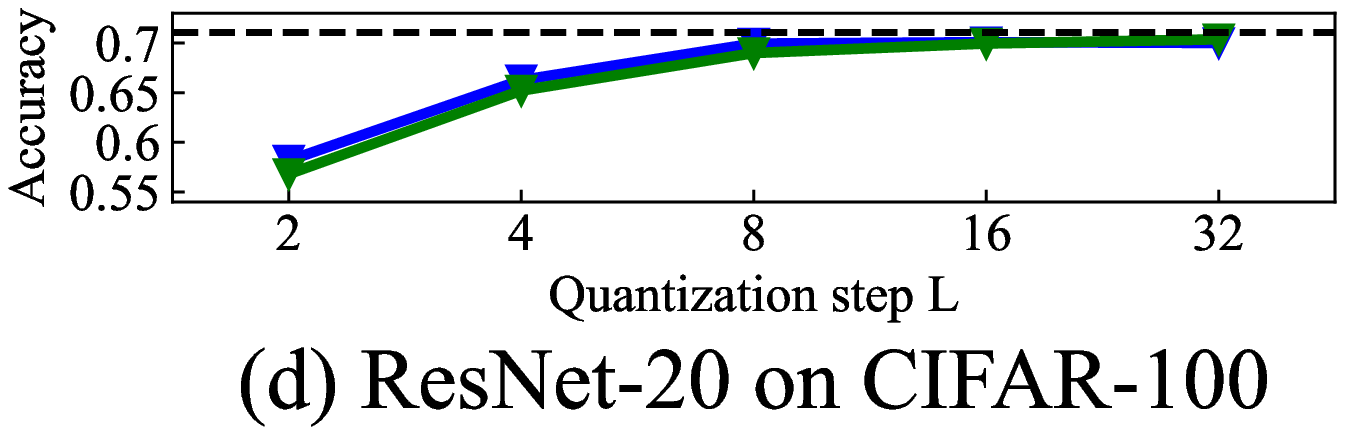}
\caption{Compare ANNs accuracy.}
\label{fig3} 
\end{figure}

We first compare the performance of ANNs with quantization clip-floor activation (green curve), ANNs with quantization clip-floor-shift activation (blue curve), and original ANNs with ReLU activation (black dotted line). Figure~\ref{fig3}(a)-(d) report the results about VGG-16 on CIFAR-10, ResNet-20 on CIFAR-10, VGG-16 on CIFAR-100 and ResNet-20 on CIFAR-100. The performance of ANNs with quantization clip-floor-shift activation is better than ANNs with quantization clip-floor activation. These two ANNs can achieve the same performance as original ANNs with ReLU activation when $L>4$.
These results demonstrate that our quantization clip-floor-shift activation function hardly affects the performance of ANN.

\subsection{Comparison with the state-of-the-art}

Table~\ref{tab:acc_convert_cifar10} compares our method with the state-of-the-art ANN-SNN conversion methods on CIFAR-10. 
As for low latency inference (T $\leq64$), our model outperforms all the other methods with the same time-step setting. For T $=32$, the accuracy of our method is slightly better than that of ANN (95.54\% vs. 95.52\%), whereas RMP, RTS, RNL, and SNNC-AP methods have accuracy loss of 33.3\%, 19.48\%, 7.42\%, and 2.01\%. Moreover, we achieve an accuracy of 93.96\% using only 4 time-steps, which is 8 times faster than SNNC-AP that takes 32 time-steps. For ResNet-20, we achieve an accuracy of 83.75\% with 4 time-steps. Notably, our ultra-low latency performance is comparable with other state-of-the-art supervised training methods, which is shown in Table~\ref{tab_comp_other_method}  of the Appendix. 

We further test the performance of our method on the large-scale dataset. Table~\ref{tab_acc_convert_imagenenew} reports the results on ImageNet, our method also outperforms the others both in terms of high accuracy and 
ultra-low latency.  For ResNet-34, the accuracy of the proposed method is 4.83\% higher than SNNC-AP and 69.28\% higher than RTS when $T=32$. When the time-steps is 16, we can still achieve an accuracy of 59.35\%. 
 For VGG-16, the accuracy of the proposed method is 4.83\% higher than SNNC-AP and 68.356\% higher than RTS when $T=32$. When the time-steps is 16, we can still achieve an accuracy of 50.97\%. 
These results demonstrate that our method outperforms the previous conversion methods.
More experimental results on CIFAR-100 is in Table~\ref{tab_acc_convert_cifar100}  of the Appendix.

\begin{figure}[t] 
\centering
\includegraphics[width=0.23\textwidth]{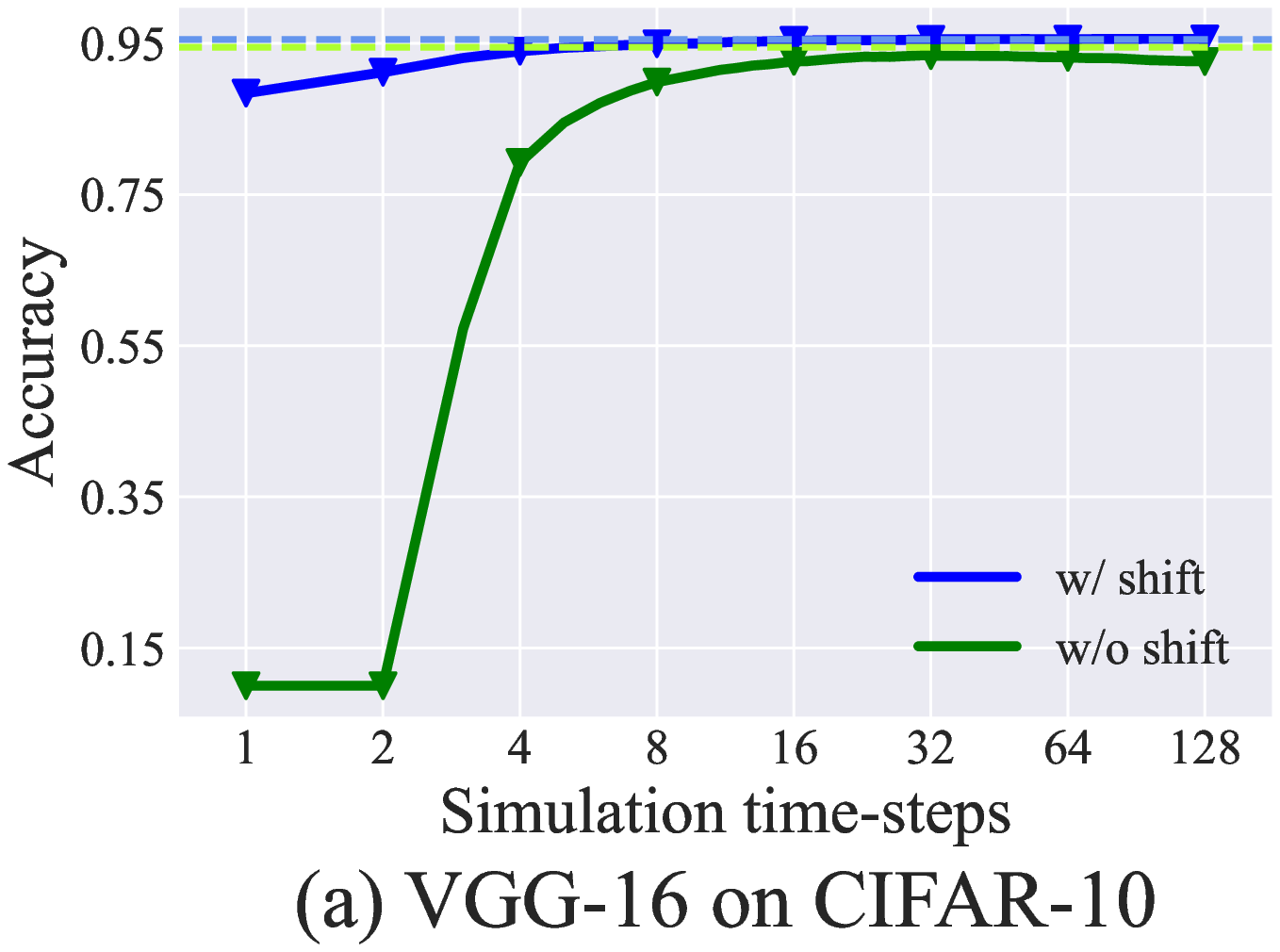}
\includegraphics[width=0.23\textwidth]{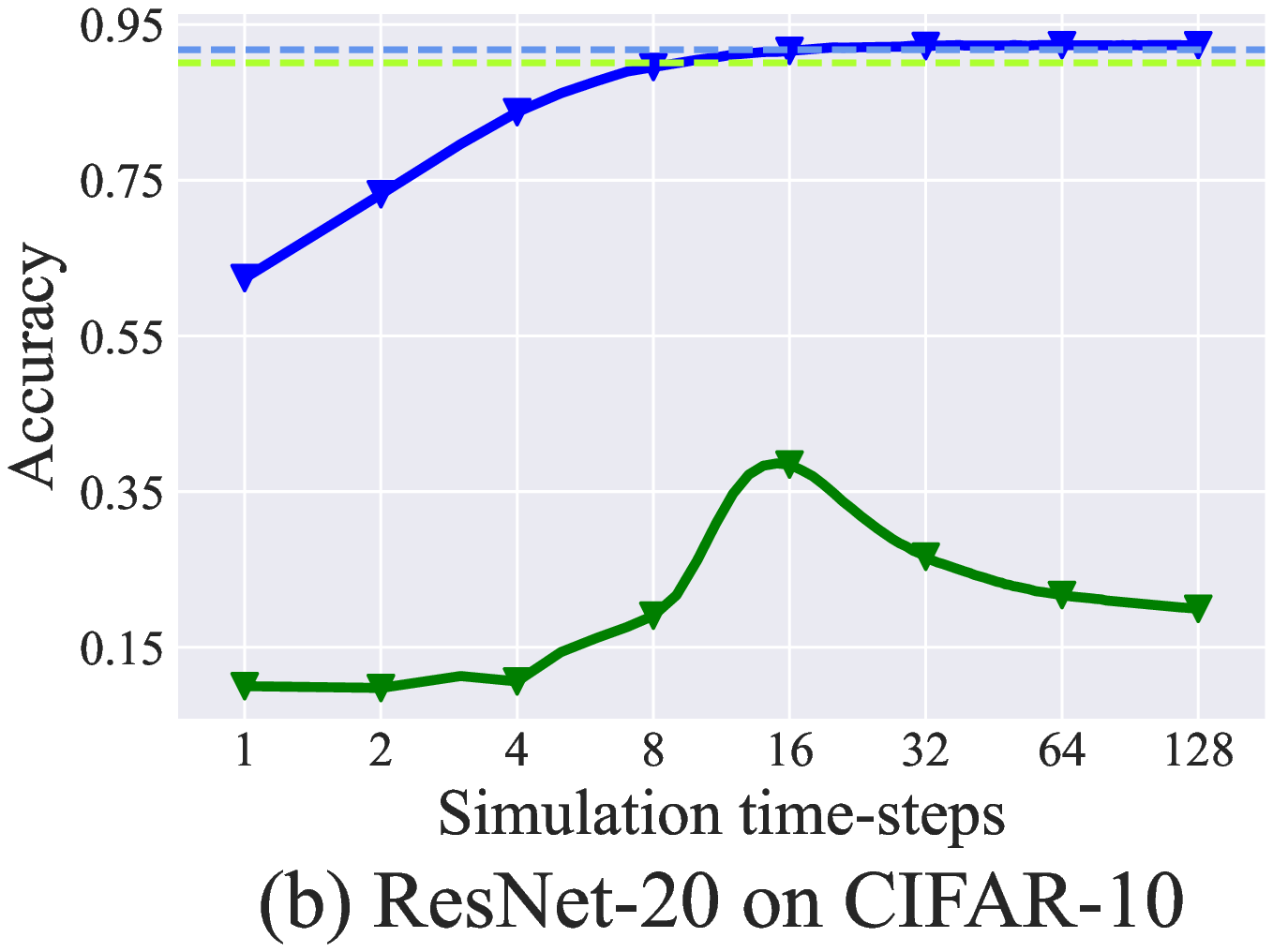}
\includegraphics[width=0.23\textwidth]{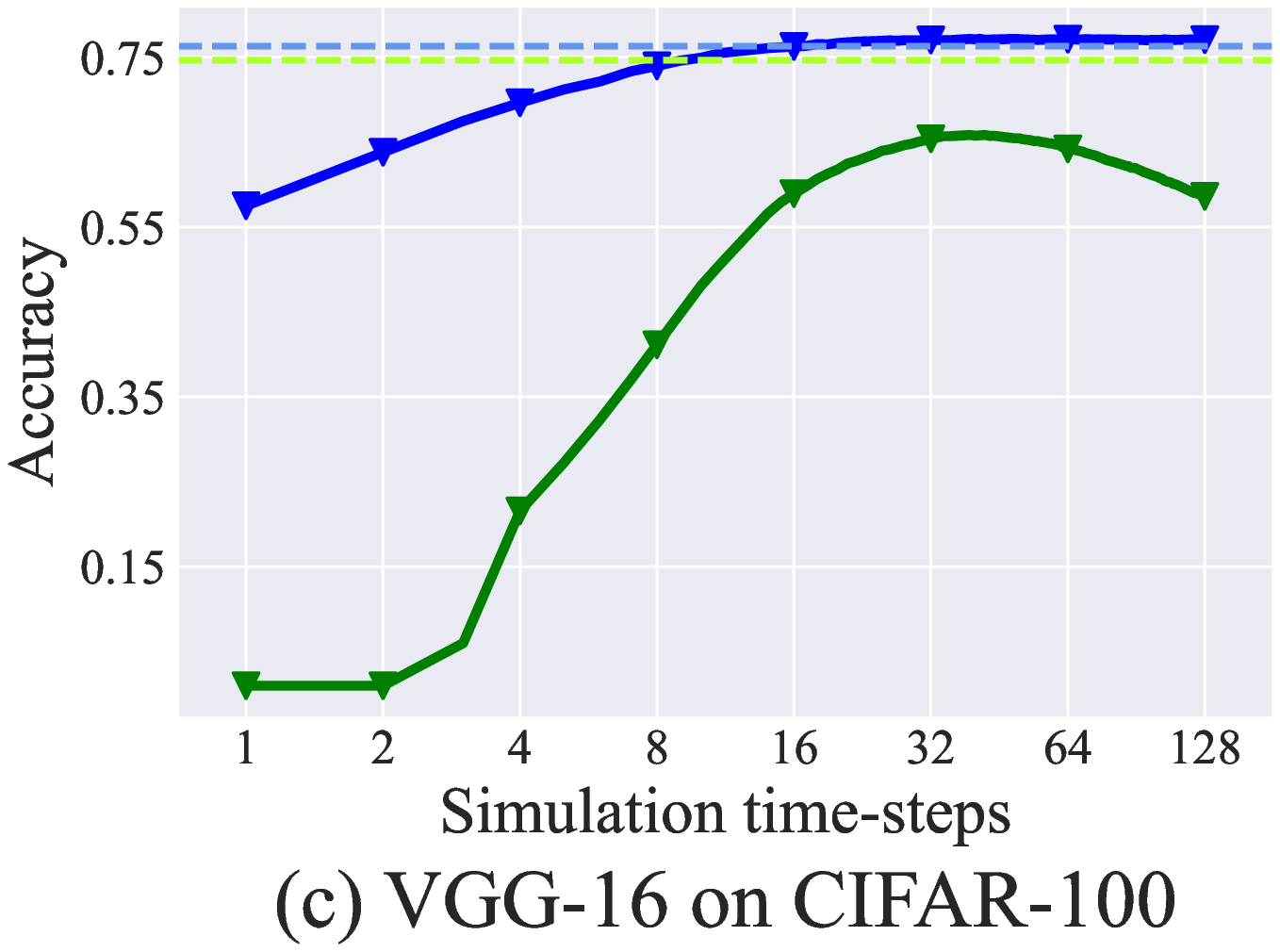}
\includegraphics[width=0.23\textwidth]{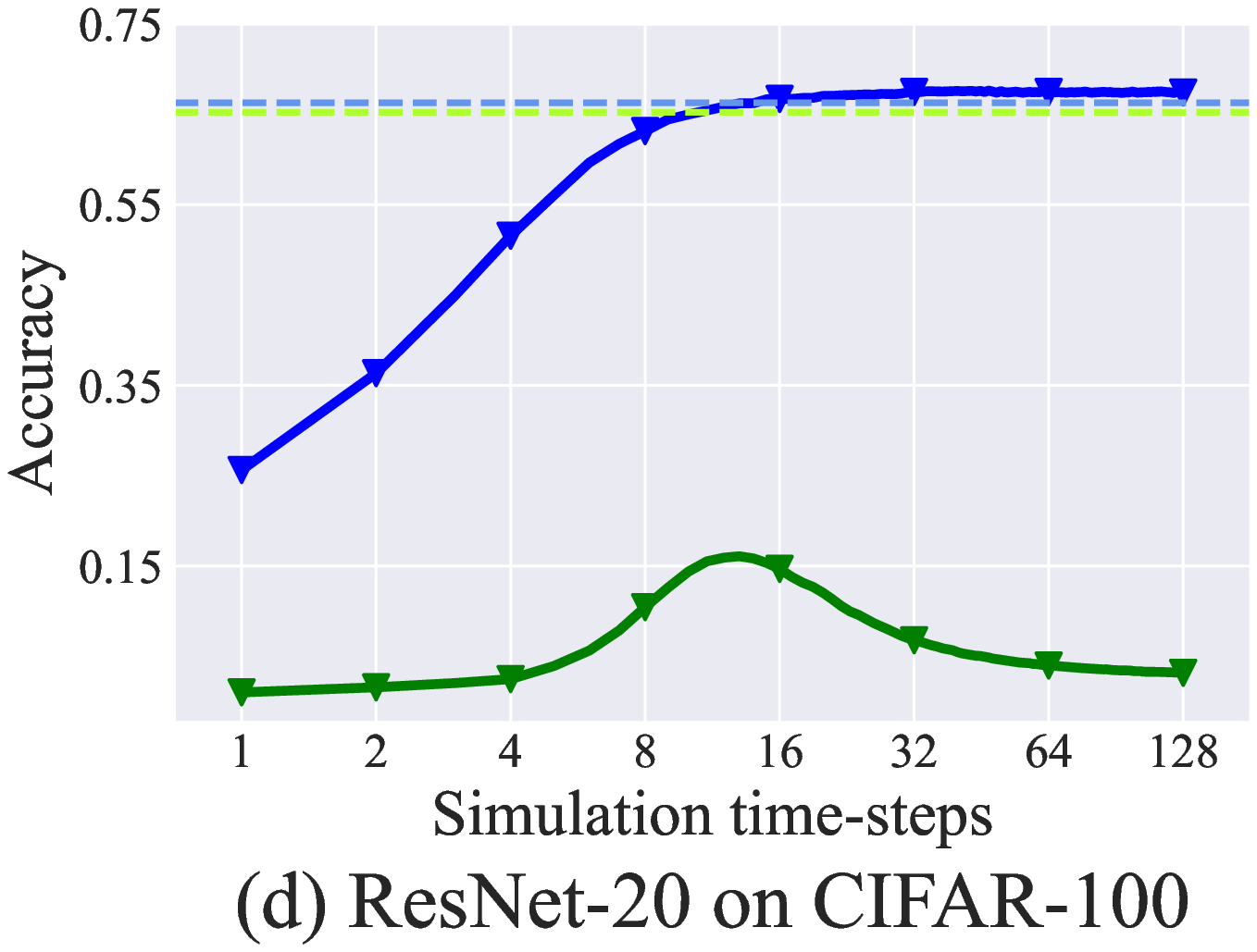}
\caption{Compare quantization clip-floor activation with/without shift term}
\label{fig4} 
\end{figure}

\begin{table}[t]
    \caption{ Comparison between the proposed method and previous works on CIFAR-10 dataset.}
    \label{tab:acc_convert_cifar10}
\centering
\renewcommand\arraystretch{1.2}
\scalebox{0.80}
{
\begin{threeparttable}
\begin{tabular}{@{}clllllllll@{}}\hline
Architecture & Method                & ANN & T=2  & T=4  & T=8  & T=16  & T=32 & T=64 & T$\geq$512\\ \hline
\multirow{6}{*}{VGG-16}  
&RMP  & 93.63\%    & -     & -     & -     & -     & 60.30\%  & 90.35\%  & 93.63\%               \\ \cline{2-10}
&TSC  & 93.63\%    & -     & -     & -     & -     & -      & 92.79\%  & 93.63\%              \\ \cline{2-10}
&RTS       & 95.72\%    & -     & -     & -     & -     & 76.24\%  & 90.64\%  & 95.73\%              \\ \cline{2-10} 
&RNL  & 92.82\%    & -     & -  & -  & 57.90\% & 85.40\% & 91.15\% & 92.95\%              \\ \cline{2-10}
&SNNC-AP  & 95.72\%    & -     & -     & -  & -  & 93.71\% & 95.14\% & 95.79\%              \\ \cline{2-10}
&\textbf{Ours}            & 95.52\% & 91.18\% & 93.96\% & 94.95\% & 95.40\% & 95.54\% & 95.55\% & 95.59\%          \\ \hline
\multirow{3}{*}{ResNet-20}
&RMP   & 91.47\%    & -     & -     & -     & -     & -     & -    & 91.36\%         \\ \cline{2-10}
&TSC  & 91.47\%    & -     & -     & -     & -     & -     & 69.38\%   & 91.42\%      \\ \cline{2-10}
&\textbf{Ours}                           & 91.77\%    & 73.20\% & 83.75\% & 89.55\% & 91.62\% & 92.24\% & 92.35\%   & 92.41\%       \\ \hline
\multirow{3}{*}{ResNet-18}
&RTS \tnote{1}  & 95.46\%    & -     & -     & -     & -     & 84.06\%   & 92.48\%    & 94.42\%  \\ \cline{2-10}
&SNNC-AP \tnote{1} & 95.46\%    & -     & -     & -     & -     & 94.78\%   & 95.30\%    & 95.45\%  \\ \cline{2-10}
&\textbf{Ours}     & 96.04\%    & 75.44\% & 90.43\% & 94.82\% & 95.92\% & 96.08\% & 96.06\%    & 96.06\%          \\ \hline
\end{tabular}
\begin{tablenotes}
					\footnotesize
					\item[1] RTS and SNNC-AP use altered ResNet-18, while ours use standard ResNet-18.
				\end{tablenotes}
\end{threeparttable}
}
\end{table}

\begin{figure}[t] 
\centering
\includegraphics[width=0.23\textwidth]{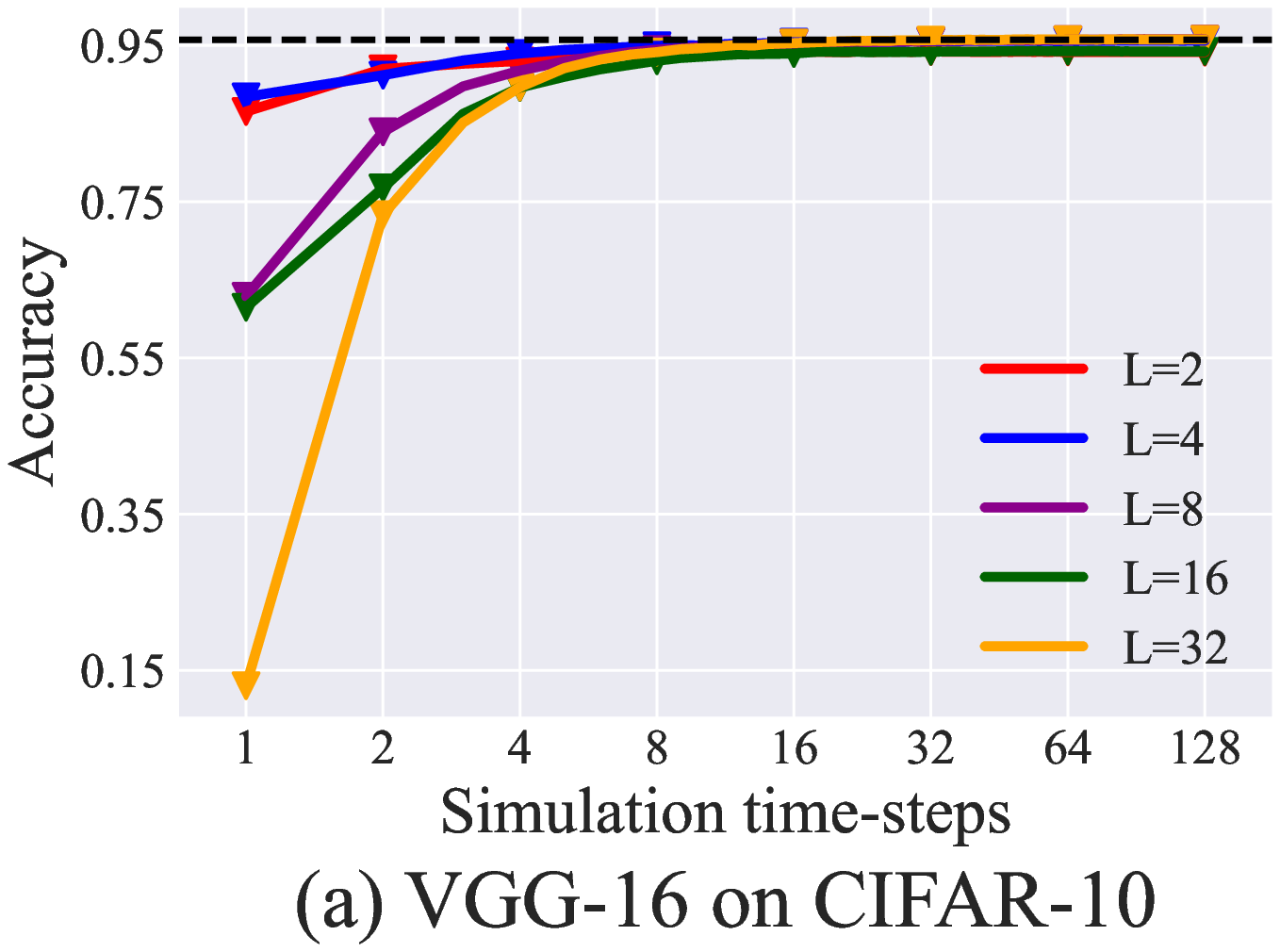} \includegraphics[width=0.23\textwidth]{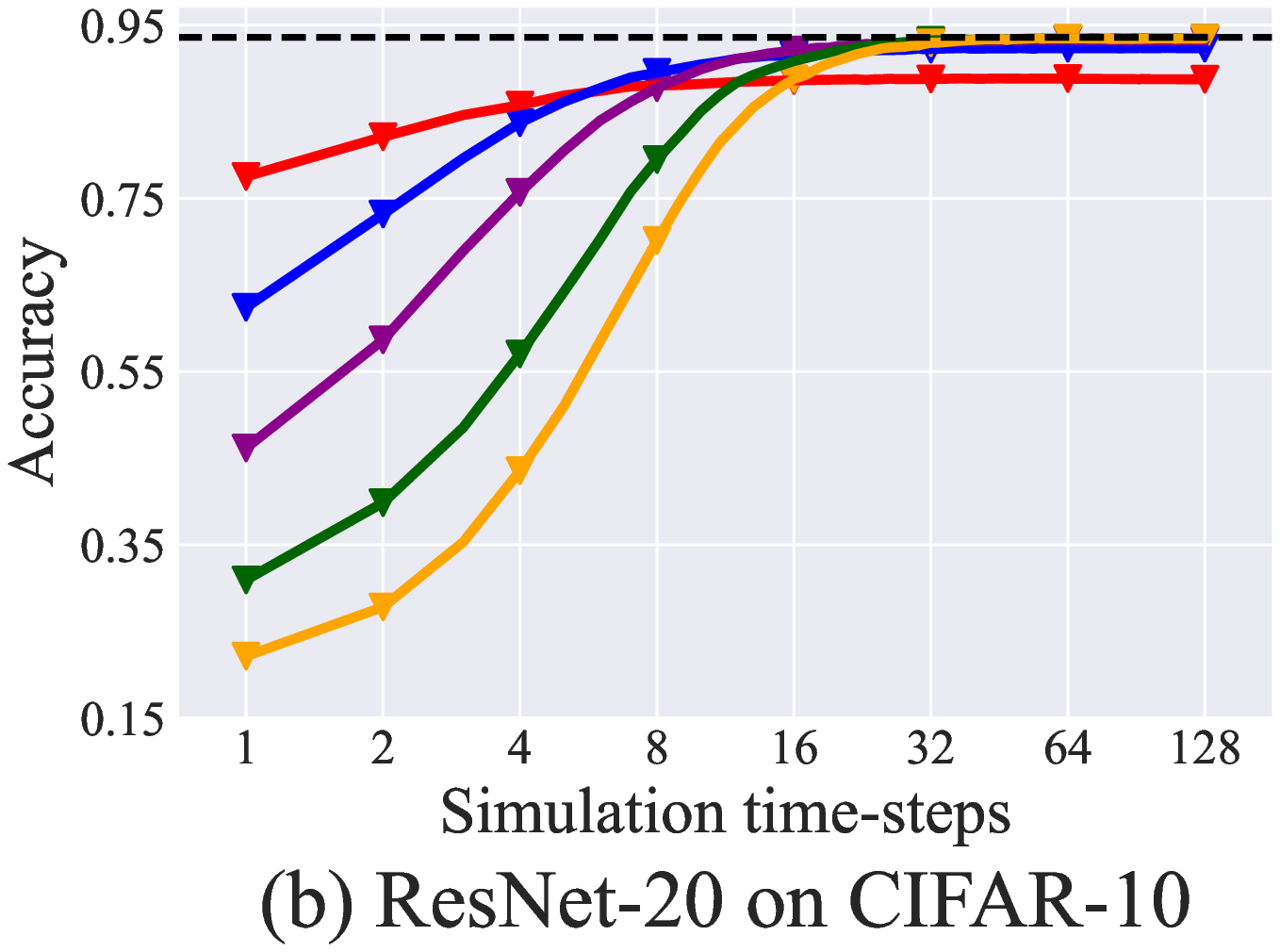} \includegraphics[width=0.23\textwidth]{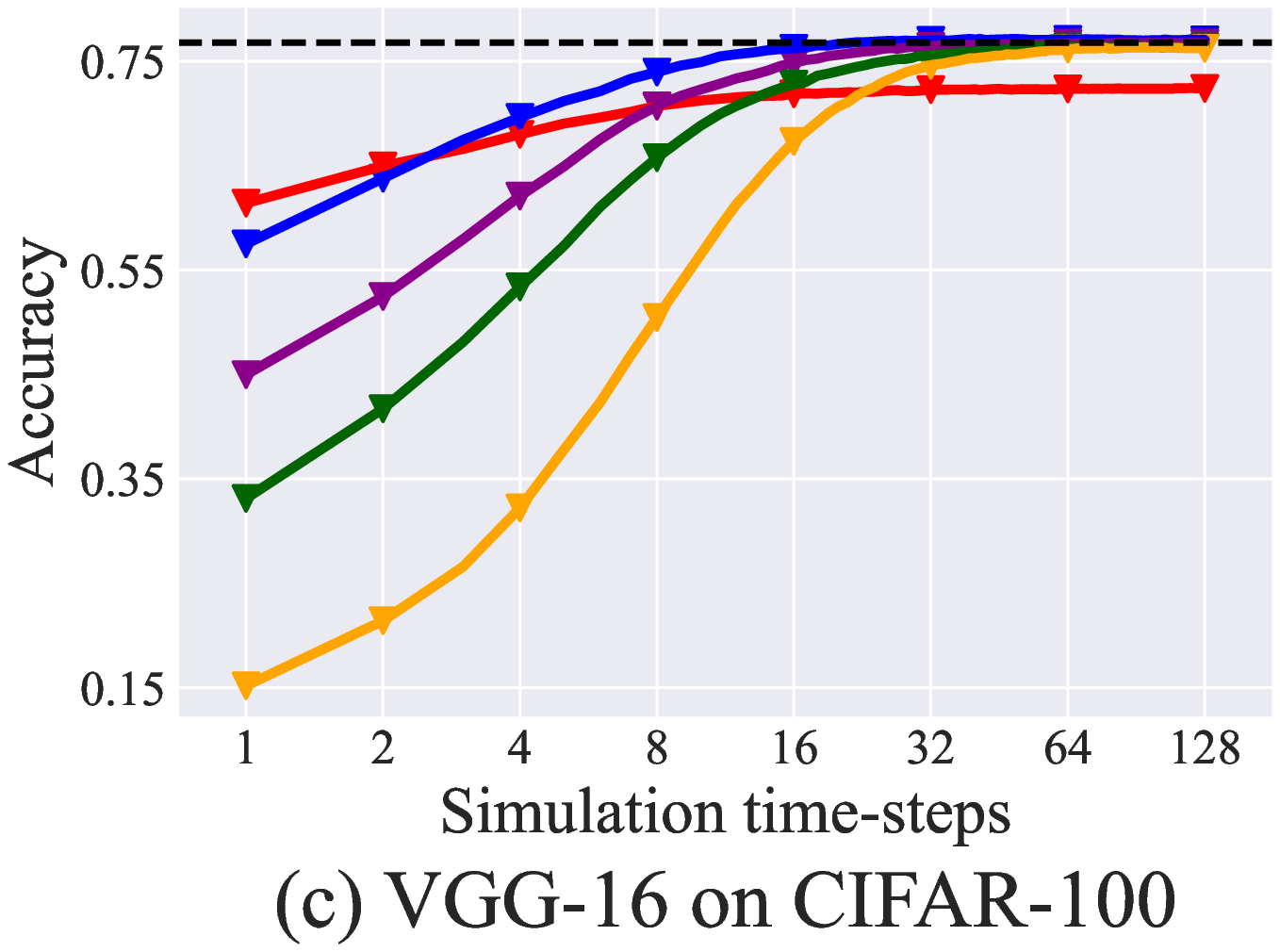}
\includegraphics[width=0.23\textwidth]{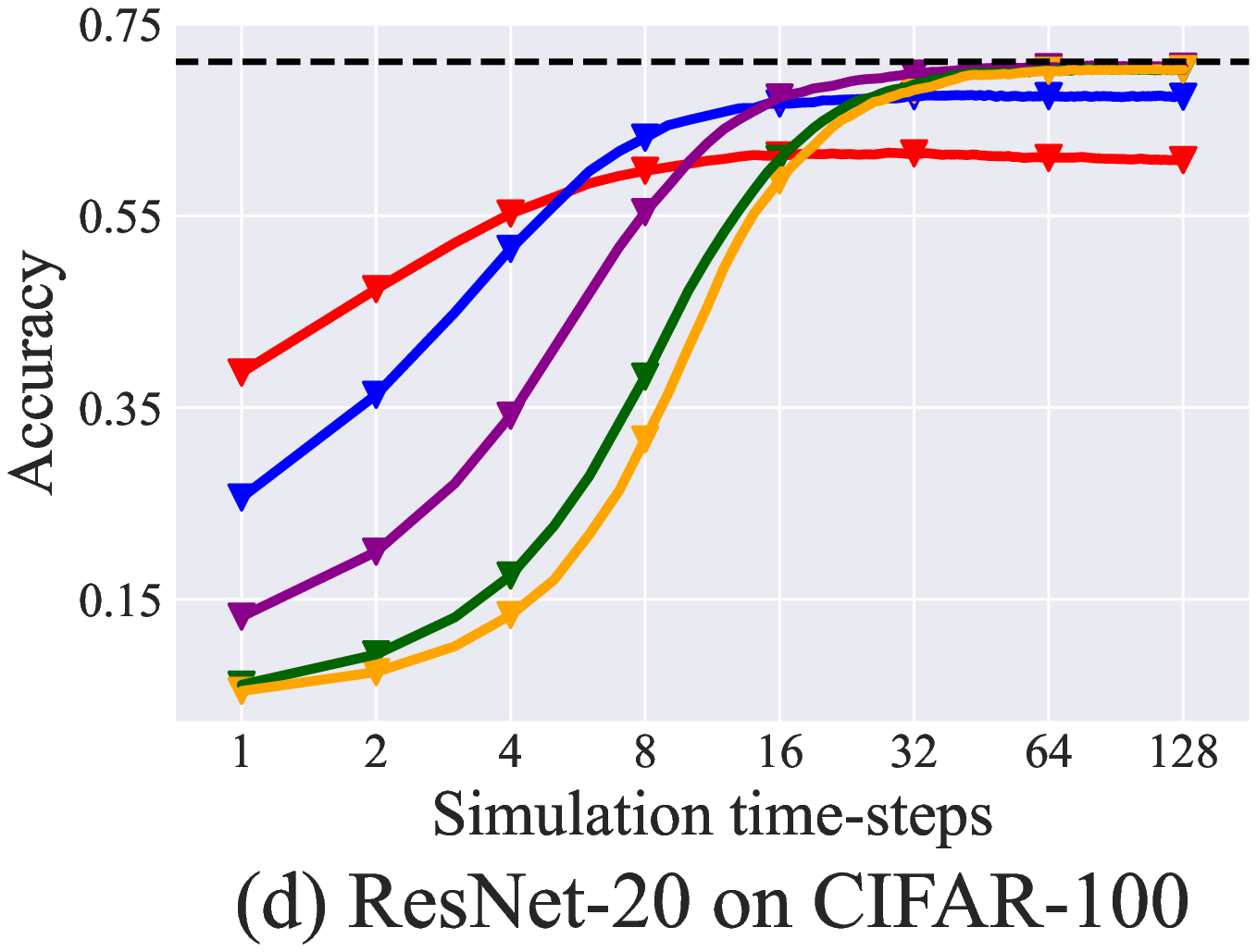}
\caption{Influence of different quantization steps}
\label{fig5} 
\end{figure}

\begin{table}[t]
    \caption{Comparison between the proposed method and previous works on ImageNet dataset.}
    \label{tab_acc_convert_imagenenew}
\centering
\renewcommand\arraystretch{1.2}
\scalebox{0.90}
{
\begin{threeparttable}
\begin{tabular}{@{}cllllllll@{}} \hline
Architecture & Method & ANN  & T=16  & T=32 & T=64 & T=128 & T=256 & T$\geq$1024 \\ \hline
\multirow{5}{*}{ResNet-34}
&RMP    & 70.64\%     & -     & -     & -     & -    & -   & 65.47\%               \\ \cline{2-9}
&TSC        & 70.64\%     & -     & -     & -     & -    &  61.48\%   &  65.10\%               \\ \cline{2-9}
&RTS       & 75.66\%     & -     &  0.09\%     & 0.12\%    & 3.19\%  & 47.11\%   & 75.08\%               \\ \cline{2-9}
&SNNC-AP       & 75.66\%     & -     & 64.54\%     & 71.12\%    & 73.45\%   & 74.61\%   & 75.45\%               \\ \cline{2-9}
&\textbf{Ours}  & 74.32\%  & 59.35\% & 69.37\% & 72.35\% & 73.15\% & 73.37\% & 73.39\% \\ \hline
\multirow{5}{*}{VGG-16}
&RMP    & 73.49\%     & -     & -     & -     & -    & 48.32\%   & 73.09\%               \\ \cline{2-9}
&TSC        & 73.49\%     & -     & -     & -     & -    &  69.71\%   &  73.46\%               \\ \cline{2-9}
&RTS       & 75.36\%     & -     &  0.114\%     & 0.118\%    & 0.122\%  & 1.81\%   & 73.88\%               \\ \cline{2-9}
&SNNC-AP       & 75.36\%     & -     & 63.64\%     & 70.69\%    & 73.32\%   & 74.23\%   & 75.32\%               \\ \cline{2-9}
&\textbf{Ours} & 74.29\%   & 50.97\%  & 68.47\% & 72.85\% & 73.97\% & 74.22\% & 74.32\%      \\ \hline
\end{tabular}
\end{threeparttable}
}
\end{table}

\subsection{Comparison of quantization clip-floor and  quantization clip-floor-shift}
Here we further compare the performance of SNNs converted from ANNs with quantization clip-floor activation and ANN with quantization clip-floor-shift activation.
In Sec.~\ref{sec:optimal}, we prove that the expectation of the conversion error reaches 0 with quantization clip-floor-shift activation, no matter whether $T$ and $L$ are the same or not.
To verify these, we set $L$ to 4 and train ANNs with quantization clip-floor activation and quantization clip-floor-shift activation, respectively. 
Figure~\ref{fig4} shows how the accuracy of converted SNNs changes with respect to the time-steps $T$. 
The accuracy of the converted SNN (green curve) from ANN with quantization clip-floor activation (green dotted line)  first increases and then decreases rapidly with the increase of time-steps, because we cannot guarantee that the conversion error is zero when $T$ is not equal to $L$. The best performance is still lower than source ANN (green dotted line).
In contrast, the accuracy of the converted SNN from ANN with quantization clip-floor-shift activation (blue curve) increases with the increase of $T$. It gets the same accuracy as source ANN (blue dotted line) when the time-steps is larger than 16.

\subsection{Effect of quantization steps L}
In our method, the quantization steps $L$ is a hyperparameter, which affects the accuracy of the converted SNN. To analyze the effect of $L$  and better determine the optimal value, we train VGG-16/ResNet-20 networks with quantization clip-floor-shift activation using different quantization steps L, including 2,4,8,16 and 32, and then converted them to SNNs. 
The experimental results on CIFAR-10/100 dataset are shown in Table~\ref{tab_quan} and Figure \ref{fig5}, where the black dotted line denotes the ANN accuracy and the colored curves represent the accuracy of the converted SNN. In order to balance the trade-off between low latency and high accuracy, we evaluate the performance of converted SNN mainly in two aspects. First, we focus on the SNN accuracy at ultra-low latency (within 4 time-steps). Second, we consider the best accuracy of SNN. It is obvious to find that the SNN accuracy at ultra-low latency decreases as $L$ increases. However, a too small $L$ will decrease the model capacity and further lead to accuracy loss.  When $L=2$, there exists a clear gap between the best accuracy of SNN and source ANN.
The best accuracy of SNN approaches source ANN when $L>4$. In conclusion, the setting of parameter $L$ mainly depends on the aims for low latency or best accuracy. The recommend quantization step $L$ is 4 or 8, which leads to high-performance converted SNN at both small time-steps and very large time-steps.

\section{Discussion and conclusion}
In this paper, we present ANN-SNN conversion method, enabling high-accuracy and ultra-low-latency deep SNNs. We propose the quantization clip-floor-shift activation to replace ReLU activation, which hardly affects the performance of ANNs and is closer to SNNs activation. Furthermore, we prove that the expected conversion error is zero, no matter whether the time-steps of SNNs and the quantization steps of ANNs is the same or not. We achieve state-of-the-art accuracy with fewer time-steps on CIFAR-10, CIFAR-100, and ImageNet datasets. Our results can benefit the implementations on neuromorphic hardware and pave the way for the large-scale application of SNNs.

Different from the work of \cite{deng2020optimal}, which adds the bias of the converted SNNs to shift the theoretical ANN-SNN curve to minimize the quantization error, we add the shift term in the quantization clip-floor activation function, and use this quantization clip-floor-shift function to train the source ANN. We show that the shift term can overcome the performance degradation problem when the time-steps and the quantization steps are not matched. Due to the unevenness error, there still exists a gap between ANN accuracy and SNN accuracy, even when $L=T$. Moreover, it is hard to achieve high-performance ANN-SNN conversion when the time-steps $T=1$. All these problems deserve further research. One advantage of conversion-based methods is that they can reduce the overall computing cost while maintaining comparable performance as source ANN. Combining the conversion-based methods and model compression may help significantly reduce the neuron activity and thus reduce energy consumptions without suffering from accuracy loss~\citep{kundu2021spike,rathi2021diet}, which is a promising direction.

\clearpage
\section*{Acknowledgement}
This work was supported by the National Natural Science Foundation of China under contracts No.62176003 and No.62088102.

\bibliography{iclr2022_conference}

\begin{thebibliography}{49}
\providecommand{\natexlab}[1]{#1}
\providecommand{\url}[1]{\texttt{#1}}
\expandafter\ifx\csname urlstyle\endcsname\relax
  \providecommand{\doi}[1]{doi: #1}\else
  \providecommand{\doi}{doi: \begingroup \urlstyle{rm}\Url}\fi

\bibitem[Bengio et~al.(2013)Bengio, L{\'e}onard, and
  Courville]{bengio2013estimating}
Yoshua Bengio, Nicholas L{\'e}onard, and Aaron Courville.
\newblock Estimating or propagating gradients through stochastic neurons for
  conditional computation.
\newblock \emph{arXiv preprint arXiv:1308.3432}, 2013.

\bibitem[Bohte et~al.(2002)Bohte, Kok, and La~Poutre]{bohte2002error}
Sander~M Bohte, Joost~N Kok, and Han La~Poutre.
\newblock Error-backpropagation in temporally encoded networks of spiking
  neurons.
\newblock \emph{Neurocomputing}, 48\penalty0 (1-4):\penalty0 17--37, 2002.

\bibitem[Bottou(2012)]{bottou2012stochastic}
L{\'e}on Bottou.
\newblock Stochastic gradient descent tricks.
\newblock In \emph{Neural networks: Tricks of the trade}, pp.\  421--436.
  Springer, 2012.

\bibitem[Cao et~al.(2015)Cao, Chen, and Khosla]{cao2015spiking}
Yongqiang Cao, Yang Chen, and Deepak Khosla.
\newblock Spiking deep convolutional neural networks for energy-efficient
  object recognition.
\newblock \emph{International Journal of Computer Vision}, 113\penalty0
  (1):\penalty0 54--66, 2015.

\bibitem[Cubuk et~al.(2019)Cubuk, Zoph, Mane, Vasudevan, and
  Le]{cubuk2019autoaugment}
Ekin~D Cubuk, Barret Zoph, Dandelion Mane, Vijay Vasudevan, and Quoc~V Le.
\newblock Autoaugment: Learning augmentation strategies from data.
\newblock In \emph{IEEE Conference on Computer Vision and Pattern Recognition},
  pp.\  113--123, 2019.

\bibitem[Davies et~al.(2018)Davies, Srinivasa, Lin, Chinya, Cao, Choday, Dimou,
  Joshi, Imam, Jain, et~al.]{davies2018loihi}
Mike Davies, Narayan Srinivasa, Tsung-Han Lin, Gautham Chinya, Yongqiang Cao,
  Sri~Harsha Choday, Georgios Dimou, Prasad Joshi, Nabil Imam, Shweta Jain,
  et~al.
\newblock Loihi: A neuromorphic manycore processor with on-chip learning.
\newblock \emph{IEEE Micro}, 38\penalty0 (1):\penalty0 82--99, 2018.

\bibitem[DeBole et~al.(2019)DeBole, Taba, Amir, Akopyan, Andreopoulos, Risk,
  Kusnitz, Otero, Nayak, Appuswamy, et~al.]{debole2019truenorth}
Michael~V DeBole, Brian Taba, Arnon Amir, Filipp Akopyan, Alexander
  Andreopoulos, William~P Risk, Jeff Kusnitz, Carlos~Ortega Otero, Tapan~K
  Nayak, Rathinakumar Appuswamy, et~al.
\newblock {T}rue{N}orth: Accelerating from zero to 64 million neurons in 10
  years.
\newblock \emph{Computer}, 52\penalty0 (5):\penalty0 20--29, 2019.

\bibitem[Deng et~al.(2009)Deng, Dong, Socher, Li, Li, and
  Fei-Fei]{deng2009imagenet}
Jia Deng, Wei Dong, Richard Socher, Li-Jia Li, Kai Li, and Li~Fei-Fei.
\newblock Imagenet: A large-scale hierarchical image database.
\newblock In \emph{IEEE Conference on Computer Vision and Pattern Recognition},
  pp.\  248--255. Ieee, 2009.

\bibitem[Deng \& Gu(2020)Deng and Gu]{deng2020optimal}
Shikuang Deng and Shi Gu.
\newblock Optimal conversion of conventional artificial neural networks to
  spiking neural networks.
\newblock In \emph{International Conference on Learning Representations}, 2020.

\bibitem[DeVries \& Taylor(2017)DeVries and Taylor]{devries2017improved}
Terrance DeVries and Graham~W Taylor.
\newblock Improved regularization of convolutional neural networks with cutout.
\newblock \emph{arXiv preprint arXiv:1708.04552}, 2017.

\bibitem[Diehl et~al.(2015)Diehl, Neil, Binas, Cook, Liu, and
  Pfeiffer]{diehl2015fast}
Peter~U Diehl, Daniel Neil, Jonathan Binas, Matthew Cook, Shih-Chii Liu, and
  Michael Pfeiffer.
\newblock Fast-classifying, high-accuracy spiking deep networks through weight
  and threshold balancing.
\newblock In \emph{International Joint Conference on Neural Networks}, pp.\
  1--8, 2015.

\bibitem[Ding et~al.(2021)Ding, Yu, Tian, and Huang]{ding2021optimal}
Jianhao Ding, Zhaofei Yu, Yonghong Tian, and Tiejun Huang.
\newblock Optimal ann-snn conversion for fast and accurate inference in deep
  spiking neural networks.
\newblock In \emph{International Joint Conference on Artificial Intelligence},
  pp.\  2328--2336, 2021.

\bibitem[Fang et~al.(2021)Fang, Yu, Chen, Huang, Masquelier, and
  Tian]{fang2021deep}
Wei Fang, Zhaofei Yu, Yanqi Chen, Tiejun Huang, Timoth{\'e}e Masquelier, and
  Yonghong Tian.
\newblock Deep residual learning in spiking neural networks.
\newblock \emph{arXiv preprint arXiv:2102.04159}, 2021.

\bibitem[Han \& Roy(2020)Han and Roy]{han2020deep}
Bing Han and Kaushik Roy.
\newblock Deep spiking neural network: Energy efficiency through time based
  coding.
\newblock In \emph{European Conference on Computer Vision}, pp.\  388--404,
  2020.

\bibitem[Han et~al.(2020)Han, Srinivasan, and Roy]{han2020rmp}
Bing Han, Gopalakrishnan Srinivasan, and Kaushik Roy.
\newblock {RMP-SNN}: Residual membrane potential neuron for enabling deeper
  high-accuracy and low-latency spiking neural network.
\newblock In \emph{IEEE Conference on Computer Vision and Pattern Recognition},
  pp.\  13558--13567, 2020.

\bibitem[He et~al.(2016)He, Zhang, Ren, and Sun]{he2016deep}
Kaiming He, Xiangyu Zhang, Shaoqing Ren, and Jian Sun.
\newblock Deep residual learning for image recognition.
\newblock In \emph{IEEE conference on Computer Vision and Pattern Recognition},
  pp.\  770--778, 2016.

\bibitem[Ho \& Chang(2020)Ho and Chang]{ho2020tcl}
Nguyen-Dong Ho and Ik-Joon Chang.
\newblock Tcl: an ann-to-snn conversion with trainable clipping layers.
\newblock \emph{arXiv preprint arXiv:2008.04509}, 2020.

\bibitem[Izhikevich(2003)]{izhikevich2003simple}
Eugene~M Izhikevich.
\newblock Simple model of spiking neurons.
\newblock \emph{IEEE Transactions on neural networks}, 14\penalty0
  (6):\penalty0 1569--1572, 2003.

\bibitem[Kheradpisheh \& Masquelier(2020)Kheradpisheh and
  Masquelier]{kheradpisheh2020temporal}
Saeed~Reza Kheradpisheh and Timoth{\'e}e Masquelier.
\newblock Temporal backpropagation for spiking neural networks with one spike
  per neuron.
\newblock \emph{International Journal of Neural Systems}, 30\penalty0
  (06):\penalty0 2050027, 2020.

\bibitem[Kim et~al.(2020)Kim, Kim, and Kim]{kim2020unifying}
Jinseok Kim, Kyungsu Kim, and Jae-Joon Kim.
\newblock Unifying activation- and timing-based learning rules for spiking
  neural networks.
\newblock In \emph{Advances in Neural Information Processing Systems}, pp.\
  19534--19544, 2020.

\bibitem[Krizhevsky et~al.(2009)Krizhevsky, Hinton,
  et~al.]{krizhevsky2009learning}
Alex Krizhevsky, Geoffrey Hinton, et~al.
\newblock Learning multiple layers of features from tiny images.
\newblock 2009.

\bibitem[Kundu et~al.(2021)Kundu, Datta, Pedram, and Beerel]{kundu2021spike}
Souvik Kundu, Gourav Datta, Massoud Pedram, and Peter~A Beerel.
\newblock Spike-thrift: Towards energy-efficient deep spiking neural networks
  by limiting spiking activity via attention-guided compression.
\newblock In \emph{Proceedings of the IEEE/CVF Winter Conference on
  Applications of Computer Vision (WACV)}, pp.\  3953--3962, 2021.

\bibitem[LeCun et~al.(1998)LeCun, Bottou, Bengio, and
  Haffner]{lecun1998gradient}
Yann LeCun, L{\'e}on Bottou, Yoshua Bengio, and Patrick Haffner.
\newblock Gradient-based learning applied to document recognition.
\newblock \emph{Proceedings of the IEEE}, 86\penalty0 (11):\penalty0
  2278--2324, 1998.

\bibitem[Lee et~al.(2020)Lee, Sarwar, Panda, Srinivasan, and
  Roy]{lee2020enabling}
Chankyu Lee, Syed~Shakib Sarwar, Priyadarshini Panda, Gopalakrishnan
  Srinivasan, and Kaushik Roy.
\newblock Enabling spike-based backpropagation for training deep neural network
  architectures.
\newblock \emph{Frontiers in Neuroscience}, 14, 2020.

\bibitem[Lee et~al.(2016)Lee, Delbruck, and Pfeiffer]{lee2016training}
Jun~Haeng Lee, Tobi Delbruck, and Michael Pfeiffer.
\newblock Training deep spiking neural networks using backpropagation.
\newblock \emph{Frontiers in Neuroscience}, 10:\penalty0 508, 2016.

\bibitem[Li et~al.(2021)Li, Deng, Dong, Gong, and Gu]{li2021free}
Yuhang Li, Shikuang Deng, Xin Dong, Ruihao Gong, and Shi Gu.
\newblock A free lunch from ann: Towards efficient, accurate spiking neural
  networks calibration.
\newblock In \emph{International Conference on Machine Learning}, pp.\
  6316--6325, 2021.

\bibitem[Loshchilov \& Hutter(2016)Loshchilov and Hutter]{loshchilov2016sgdr}
Ilya Loshchilov and Frank Hutter.
\newblock Sgdr: Stochastic gradient descent with warm restarts.
\newblock In \emph{International Conference on Learning Representations}, 2016.

\bibitem[Maass(1997)]{maas1997networks}
Wolfgang Maass.
\newblock {Networks of spiking neurons: the third generation of neural network
  models}.
\newblock \emph{Neural Networks}, 10\penalty0 (9):\penalty0 1659--1671, 1997.

\bibitem[Massa et~al.(2020)Massa, Marchisio, Martina, and
  Shafique]{massa2020efficient}
Riccardo Massa, Alberto Marchisio, Maurizio Martina, and Muhammad Shafique.
\newblock An efficient spiking neural network for recognizing gestures with a
  {DVS} camera on the {Loihi} neuromorphic processor.
\newblock In \emph{International Joint Conference on Neural Networks}, pp.\
  1--9, 2020.

\bibitem[McCulloch \& Pitts(1943)McCulloch and Pitts]{mcculloch1943logical}
Warren~S McCulloch and Walter Pitts.
\newblock A logical calculus of the ideas immanent in nervous activity.
\newblock \emph{The Bulletin of Mathematical Biophysics}, 5\penalty0
  (4):\penalty0 115--133, 1943.

\bibitem[Merolla et~al.(2014)Merolla, Arthur, Alvarez-Icaza, Cassidy, Sawada,
  Akopyan, Jackson, Imam, Guo, Nakamura, et~al.]{merolla2014million}
Paul~A Merolla, John~V Arthur, Rodrigo Alvarez-Icaza, Andrew~S Cassidy, Jun
  Sawada, Filipp Akopyan, Bryan~L Jackson, Nabil Imam, Chen Guo, Yutaka
  Nakamura, et~al.
\newblock A million spiking-neuron integrated circuit with a scalable
  communication network and interface.
\newblock \emph{Science}, 345\penalty0 (6197):\penalty0 668--673, 2014.

\bibitem[Neftci et~al.(2019)Neftci, Mostafa, and Zenke]{neftci2019surrogate}
Emre~O Neftci, Hesham Mostafa, and Friedemann Zenke.
\newblock Surrogate gradient learning in spiking neural networks: Bringing the
  power of gradient-based optimization to spiking neural networks.
\newblock \emph{IEEE Signal Processing Magazine}, 36\penalty0 (6):\penalty0
  51--63, 2019.

\bibitem[Pei et~al.(2019)Pei, Deng, Song, Zhao, Zhang, Wu, Wang, Zou, Wu, He,
  et~al.]{pei2019towards}
Jing Pei, Lei Deng, Sen Song, Mingguo Zhao, Youhui Zhang, Shuang Wu, Guanrui
  Wang, Zhe Zou, Zhenzhi Wu, Wei He, et~al.
\newblock Towards artificial general intelligence with hybrid tianjic chip
  architecture.
\newblock \emph{Nature}, 572\penalty0 (7767):\penalty0 106--111, 2019.

\bibitem[Qiao et~al.(2015)Qiao, Mostafa, Corradi, Osswald, Stefanini,
  Sumislawska, and Indiveri]{qiao2015reconfigurable}
Ning Qiao, Hesham Mostafa, Federico Corradi, Marc Osswald, Fabio Stefanini,
  Dora Sumislawska, and Giacomo Indiveri.
\newblock A reconfigurable on-line learning spiking neuromorphic processor
  comprising 256 neurons and 128{K} synapses.
\newblock \emph{Frontiers in neuroscience}, 9:\penalty0 141, 2015.

\bibitem[Rathi \& Roy(2021)Rathi and Roy]{rathi2021diet}
Nitin Rathi and Kaushik Roy.
\newblock Diet-snn: A low-latency spiking neural network with direct input
  encoding and leakage and threshold optimization.
\newblock \emph{IEEE Transactions on Neural Networks and Learning Systems},
  2021.

\bibitem[Rathi et~al.(2019)Rathi, Srinivasan, Panda, and
  Roy]{rathi2019enabling}
Nitin Rathi, Gopalakrishnan Srinivasan, Priyadarshini Panda, and Kaushik Roy.
\newblock Enabling deep spiking neural networks with hybrid conversion and
  spike timing dependent backpropagation.
\newblock In \emph{International Conference on Learning Representations}, 2019.

\bibitem[Roy et~al.(2019)Roy, Jaiswal, and Panda]{roy2019towards}
Kaushik Roy, Akhilesh Jaiswal, and Priyadarshini Panda.
\newblock {Towards spike-based machine intelligence with neuromorphic
  computing}.
\newblock \emph{Nature}, 575\penalty0 (7784):\penalty0 607--617, 2019.

\bibitem[Rueckauer et~al.(2016)Rueckauer, Lungu, Hu, and
  Pfeiffer]{rueckauer2016theory}
Bodo Rueckauer, Iulia-Alexandra Lungu, Yuhuang Hu, and Michael Pfeiffer.
\newblock Theory and tools for the conversion of analog to spiking
  convolutional neural networks.
\newblock \emph{arXiv preprint arXiv:1612.04052}, 2016.

\bibitem[Rueckauer et~al.(2017)Rueckauer, Lungu, Hu, Pfeiffer, and
  Liu]{rueckauer2017conversion}
Bodo Rueckauer, Iulia-Alexandra Lungu, Yuhuang Hu, Michael Pfeiffer, and
  Shih-Chii Liu.
\newblock Conversion of continuous-valued deep networks to efficient
  event-driven networks for image classification.
\newblock \emph{Frontiers in Neuroscience}, 11:\penalty0 682, 2017.

\bibitem[Russakovsky et~al.(2015)Russakovsky, Deng, Su, Krause, Satheesh, Ma,
  Huang, Karpathy, Khosla, Bernstein, Berg, and Fei-Fei]{ILSVRC15}
Olga Russakovsky, Jia Deng, Hao Su, Jonathan Krause, Sanjeev Satheesh, Sean Ma,
  Zhiheng Huang, Andrej Karpathy, Aditya Khosla, Michael Bernstein,
  Alexander~C. Berg, and Li~Fei-Fei.
\newblock {ImageNet Large Scale Visual Recognition Challenge}.
\newblock \emph{International Journal of Computer Vision (IJCV)}, 115\penalty0
  (3):\penalty0 211--252, 2015.
\newblock \doi{10.1007/s11263-015-0816-y}.

\bibitem[Sengupta et~al.(2019)Sengupta, Ye, Wang, Liu, and
  Roy]{sengupta2019going}
Abhronil Sengupta, Yuting Ye, Robert Wang, Chiao Liu, and Kaushik Roy.
\newblock Going deeper in spiking neural networks: {VGG} and residual
  architectures.
\newblock \emph{Frontiers in Neuroscience}, 13:\penalty0 95, 2019.

\bibitem[Simonyan \& Zisserman(2014)Simonyan and Zisserman]{simonyan2014very}
Karen Simonyan and Andrew Zisserman.
\newblock Very deep convolutional networks for large-scale image recognition.
\newblock \emph{arXiv preprint arXiv:1409.1556}, 2014.

\bibitem[Singh et~al.(2021)Singh, Sarma, Lu, Sengupta, Narayanan, and
  Das]{singh2021gesture}
Sonali Singh, Anup Sarma, Sen Lu, Abhronil Sengupta, Vijaykrishnan Narayanan,
  and Chita~R Das.
\newblock Gesture-snn: Co-optimizing accuracy, latency and energy of snns for
  neuromorphic vision sensors.
\newblock In \emph{IEEE/ACM International Symposium on Low Power Electronics
  and Design}, pp.\  1--6, 2021.

\bibitem[St{\"o}ckl \& Maass(2021)St{\"o}ckl and Maass]{stockl2021optimized}
Christoph St{\"o}ckl and Wolfgang Maass.
\newblock Optimized spiking neurons can classify images with high accuracy
  through temporal coding with two spikes.
\newblock \emph{Nature Machine Intelligence}, 3\penalty0 (3):\penalty0
  230--238, 2021.

\bibitem[Tavanaei et~al.(2019)Tavanaei, Ghodrati, Kheradpisheh, Masquelier, and
  Maida]{tavanaei2019deep}
Amirhossein Tavanaei, Masoud Ghodrati, Saeed~Reza Kheradpisheh, Timoth{\'e}e
  Masquelier, and Anthony Maida.
\newblock Deep learning in spiking neural networks.
\newblock \emph{Neural Networks}, 111:\penalty0 47--63, 2019.

\bibitem[Wu et~al.(2018)Wu, Deng, Li, Zhu, and Shi]{wu2018STBP}
Yujie Wu, Lei Deng, Guoqi Li, Jun Zhu, and Luping Shi.
\newblock Spatio-temporal backpropagation for training high-performance spiking
  neural networks.
\newblock \emph{Frontiers in Neuroscience}, 12:\penalty0 331, 2018.

\bibitem[Wu et~al.(2019)Wu, Deng, Li, Zhu, Xie, and Shi]{wu2019direct}
Yujie Wu, Lei Deng, Guoqi Li, Jun Zhu, Yuan Xie, and Luping Shi.
\newblock Direct training for spiking neural networks: Faster, larger, better.
\newblock In \emph{AAAI Conference on Artificial Intelligence}, pp.\
  1311--1318, 2019.

\bibitem[Zenke \& Vogels(2021)Zenke and Vogels]{zenke2021remarkable}
Friedemann Zenke and Tim~P Vogels.
\newblock The remarkable robustness of surrogate gradient learning for
  instilling complex function in spiking neural networks.
\newblock \emph{Neural Computation}, 33\penalty0 (4):\penalty0 899--925, 2021.

\bibitem[Zhang \& Li(2020)Zhang and Li]{zhang2020temporal}
Wenrui Zhang and Peng Li.
\newblock Temporal spike sequence learning via backpropagation for deep spiking
  neural networks.
\newblock In \emph{Advances in Neural Information Processing Systems}, pp.\
  12022--12033, 2020.

\end{thebibliography}
\bibliographystyle{iclr2022_conference}

\newpage

\renewcommand{\thefigure}{S\arabic{figure}}
\setcounter{figure}{0}

\renewcommand{\thetable}{S\arabic{table}}
\setcounter{table}{0}

\renewcommand{\theequation}{S\arabic{equation}}
\setcounter{equation}{0}

\appendix
\section{Appendix}
\subsection{network structure and training configurations}
Before training ANNs, we first replace max-pooling with average-pooling and then replace the ReLU activation with the proposed quantization clip-floor-shift activation (\Eqref{annnew2}). 
After training, we copy all weights from the source ANN to the converted SNN, and set the threshold $\theta^l$ in each layer of the converted SNN equal to the maximum activation value $\lambda^l$ of the source ANN in the same layer.
Besides, we set the initial membrane potential $\bm{v}^l(0)$ in converted SNN as $\bm{\theta^l}/2$ to match the optimal shift $\bm{\varphi}=\frac{1}{2}$ of quantization clip-floor-shift activation in the source ANN.

Despite the common data normalization, we use some data pre-processing techniques. For CIFAR datasets, we resize the images into $32\times32$, and for ImageNet dataset, we resize the image into $224\times224$. Besides, we use random crop images, Cutout~\citep{devries2017improved} and AutoAugment~\citep{cubuk2019autoaugment} for all datasets.

We use the Stochastic Gradient Descent optimizer~\citep{bottou2012stochastic} with a momentum parameter of 0.9. The initial learning rate is set to 0.1 for CIFAR-10 and ImageNet, and 0.02 for CIFAR-100. A cosine decay scheduler~\citep{loshchilov2016sgdr} is used to adjust the learning rate. We apply a  $5 \times 10^{-4}$ weight decay for CIFAR datasets while applying a  $1 \times 10^{-4}$ weight decay for ImageNet. We train all models for 300 epochs. The quantization steps $L$ is set to 4 when training all the networks on CIFAR-10, and VGG-16, ResNet-18 on CIFAR-100 dataset. When training ResNet-20 on CIFAR-100, the parameter $L$ is set to 8. When training ResNet-34 and VGG-16 on ImageNet, the parameter $L$ is set to 8, 16, respectively. We use constant input when evaluating the converted SNNs. 

\subsection{Introduction of Datasets}
\textbf{CIFAR-10.} The CIFAR-10 dataset~\citep{krizhevsky2009learning} consists of 60000 $32 \times 32$ images in 10 classes. There are 50000 training images and 10000 test images.

\textbf{CIFAR-100.} The CIFAR-100 dataset~\citep{krizhevsky2009learning} consists of 60000 $32 \times 32$ images in 100 classes. There are 50000 training images and 10000 test images.

\textbf{ImageNet.} We use the ILSVRC 2012 dataset~\citep{ILSVRC15}, which consists 1,281,167 training images and 50000 testing images.

\subsection{Derivation of Equation 12 and Proof of Theorem 2}
\textbf{Derivation of \Eqref{snn_est}}

Similar to $,$ , We define
\begin{align}
    \bm{u}^l(t)=\bm{W}^l \bm{x}^{l-1}(t).
\end{align}
We use $u^l_i(t)$ and $z^l_i$ to denote the $i$-th element in vector $\bm{u}^l(t)$ and $\bm{z}^l$, respectively. To derive \Eqref{snn_est}, some extra assumptions on the relationship between ANN activation value and SNN postsynaptic potentials are needed,
which are showed in \Eqref{asp}.

\begin{align}
\label{asp}
\begin{cases}
    & \mathrm{if}~z^l_i<0,~~\mathrm{then}~\forall t~u^l_i(t)<0, \\
    & \mathrm{if}~0\leqslant z^l_i \leqslant  \theta_l,~~\mathrm{then}~\forall t~0 \leqslant u^l_i(t) \leqslant  \theta_l, \\
    & \mathrm{if}~z^l_i>\theta_l,~~\mathrm{then}~\forall t~u^l_i(t)>\theta_l.
\end{cases}
\end{align}

With the assumption above, we can discuss the firing behavior of the neurons in each time-step. When $z^l_i<0$ or $z^l_i>\theta_l$, the neuron will  never fire or fire all the time-steps, which means $\phi^l_i(T)=0$ or $\phi^l_i(T)=\theta^l$. In this situation, we can use a clip function to denote $\phi^l_i(T)$.
\begin{align}
\label{s3}
    \phi^l_i(T) &= \mathrm{clip}(z^l_i, 0, \theta^l).
\end{align}
When $0<z^l_i<\theta_l$, every input from the presynaptic neuron in SNNs falls into $ [0, \theta^l] $, then we have $\forall t,~v^l_i(t)\in[0, \theta]$. We can rewrite \Eqref{postpoten} into the following equation.
\begin{align}
\label{s4}
\frac{\phi^l_i(T) T}{\theta^l} = \frac{z^l_i T + v^l_i(0)}{\theta^l} - \frac{v^l_i(T)}{\theta^l}. 
\end{align}

Considering that $\frac{\phi^l_i(T) T}{\theta^l} = \sum_{t=1}^{T} s^l_i(t) \in \mathbb{N}$ and $0<\frac{v^l_i(T)}{\theta^l}<1$, \Eqref{s4} is changed to:
\begin{align}
\phi^l_i(T) &= \frac{\theta^l}{T} \left \lfloor \frac{z^l_i T + v^l_i(0)}{\theta^l} \right \rfloor.
\end{align}

We combine these two situations (\Eqref{s3} and \Eqref{s4}), and we have:
\begin{align}
    \bm{\phi}^l(T) = \theta^{l}~\mathrm{clip} \left( \frac{1}{T}\left \lfloor \frac{\bm{z}^{l} T  + \bm{v}^{l}(0)}{\theta^{l}} \right \rfloor, 0, 1 \right).
\end{align}

\textbf{Proof of Theorem 2}

Before prove Theorem \ref{t2}, we first introduce Lemma \ref{l1}.
\begin{lemma}
\label{l1}
If random variable $x\in [0, \theta]$ is uniformly distributed in every small interval $[m_{t}, m_{t+1}]$ with the probability density function $p_t$ ($t=0,1,...,T$), where $m_{0}=0, m_{T+1}=\theta, m_t=\frac{(t-\frac{1}{2})\theta}{T}$ for $t=1,2,...,T$, $p_0=p_T$, we can conclude that
\begin{align}
    \mathbb{E}_x \left( x - \frac{\theta}{T} \left \lfloor \frac{Tx}{\theta} + \frac{1}{2} \right \rfloor \right) = 0.
\end{align}
\end{lemma}

\begin{proof}
\begin{align}
    & \mathbb{E}_{x} \left( x - \frac{\theta}{T} \left \lfloor \frac{ T x}{\theta} + \frac{1}{2} \right \rfloor  \right) 
    = \int_{0}^{\theta/2T}  p_0  \left(x- \frac{\theta}{T} \left \lfloor \frac{xT}{\theta}+\frac{1}{2} \right \rfloor \right)~\mathrm{d}x \nonumber \\
    & + \sum_{t=1}^{T-1} \int_{(2t-1)\theta/2T}^{(2t+1)\theta/2T} p_t \left (x - \frac{\theta}{T} \left \lfloor \frac{xT}{\theta}+\frac{1}{2} \right \rfloor  \right)~\mathrm{d}x \nonumber\\
    & + \int_{(2T-1)\theta/2T}^{\theta} p_T \left(x - \frac{\theta}{T} \left \lfloor \frac{xT}{\theta}+\frac{1}{2} \right \rfloor \right)~\mathrm{d}x \nonumber \\
    & = p_0 \int_{0}^{\theta/2T} x~\mathrm{d}x + \sum_{t=1}^{T-1} p_t \int_{(2t-1)\theta/2T}^{(2t+1)\theta/2T} (x-\frac{t\theta}{T})~\mathrm{d}x + p_T \int_{(2T-1)\theta/2T}^{\theta} (x-\theta)~\mathrm{d}x \nonumber\\
    & = p_0 \frac{\theta^2}{8T^2} + 0 - p_T \frac{\theta^2}{8T^2} = (p_0-p_T)\frac{\theta^2}{8T^2} = 0.
\end{align}

\end{proof}

\textbf{Theorem 2.} \textit{An ANN with activation function (\ref{annnew2}) is converted to an SNN with the same weights. If $\theta^l=\lambda^l$, $\bm{v}^{l}(0)=\theta^l\bm{\varphi}$, 
then for arbitrary $T$ and $L$, the expectation of conversion error reaches $\bm{0}$ when the shift term $\bm{\varphi}$ in source ANN is $\bm{\frac{1}{2}}$.}
\begin{align}
\forall\ T,L \quad
\left.\mathbb{E}_z \left ( {\widetilde{\bm{Err}}}^l \right ) \right |_{\bm{\varphi} = \bm{\frac{1}{2}}} &= \bm{0}. 
\end{align}

\begin{proof}

\begin{align}
\left.\mathbb{E}_z \left ( {\widetilde{\bm{Err}}}^l \right ) \right |_{\bm{\varphi} = \bm{\frac{1}{2}}} &= 
 \mathbb{E}_{z} \left ( \frac{\theta^l}{T} \left \lfloor \frac{\bm{z}^l T + \bm{v}^l(0)}{\theta^l} \right \rfloor - \frac{\lambda^l}{L} \left \lfloor \frac{\bm{z}^l L}{\lambda}+\bm{\varphi} \right \rfloor \right ).  
\end{align}

As every element in vector $\bm{z}$ is identical, we only need to consider one element.
\begin{align}
    & \mathbb{E}_{z_i} \left ( \frac{\theta^l}{T} \left \lfloor \frac{z_i^l T + v^l_i(0)}{\theta^l} \right \rfloor - \frac{\lambda^l}{L} \left \lfloor \frac{z_i^l L}{\lambda}+\varphi_i \right \rfloor \right) \nonumber \\
    & = \mathbb{E}_{z_i} \left ( \frac{\theta^l}{T} \left \lfloor \frac{z_i^l T + v^l_i(0)}{\theta^l} \right \rfloor - z_i^l \right ) + \mathbb{E}_{z_i} \left (z_i^l - \frac{\lambda^l}{L} \left \lfloor \frac{z_i^l L}{\lambda}+\varphi_i \right \rfloor \right).
\end{align}

According to Lemma \ref{l1}, we have 
\begin{align}
   & \left. \mathbb{E}_{z_i} \left ( \frac{\theta^l}{T} \left \lfloor \frac{z_i^l T + v^l_i(0)}{\theta^l} \right \rfloor - z_i^l \right ) \right |_{v_i^l(0)=1/2} = 0, \\
   &  \left. \mathbb{E}_{z_i} \left (z_i^l - \frac{\lambda^l}{L} \left \lfloor \frac{z_i^l L}{\lambda}+\varphi_i \right \rfloor \right) \right |_{\varphi=1/2} = 0.
\end{align}

Thus the sum of both terms also equals zero.
\end{proof}

\subsection{Comparison of the methods with or without dynamic threshold on the CIFAR-100 dataset}
In this paper we use a training parameter $\lambda^l$ to decide the maximum value of ANN activation. The previous works suggested to set the maximum value of ANN activation after training as the threshold. If we  set  $\theta^l = \max_{\bm{s} \in \{0,1\}^n} \left ( \max(\theta^{l-1}\bm{W}^l\bm{s}) \right)$, the situation of fewer spikes as expected never happens, as we can prove that $\bm{v}^l(T)< \theta^l$ (see Theorem 3). Despite this, there still exists the situation of more spikes as expected. An example is given in Figure~\ref{figS1}. Here we consider the same example as in Figure~\ref{fig1}. In source ANN, we suppose that two analog neurons in layer $l-1$ are connected to an analog neuron in layer $l$ with weights 2 and -2, and the output vector $\bm{a}^{l-1}$ of neurons in layer $l-1$ is $[0.6, 0.4]$. Besides, in converted SNN, we suppose that the two spiking neurons in layer $l-1$ fire 3 spikes and 2 spikes in 5 time-steps (T=5), respectively, and the threshold $\theta^{l-1}=1$. Thus, $\bm{\phi}^{l-1}(T)=\frac{\sum_{i=1}^{T} \bm{s}^{l-1}(i)}{T}\theta^{l-1}=[0.6, 0.4]$. According to  \Eqref{ann}, the ANN output $\bm{a}^{l}=\bm{W}^{l}\bm{a}^{l-1}=[2,-2] [0.6, 0.4]^{T}=0.4$. As for SNN, we suppose that the presynaptic neurons fires at $t=1,2,3$ and $t=4,5$, respectively. Even through we set the threshold $\theta^{l}=1$ to the maximum activation 2, the postsynaptic neuron will fire three spikes at $t=1,2,3$, and  $\bm{\phi}^{l}(T)=0.6>\bm{a}^{l}$. 

Besides, setting $\max_{\bm{s} \in \{0,1\}^n} \left ( \max(\theta^{l-1}\bm{W}^l\bm{s}) \right)$ as the threshold brings two other problems. First, the spiking neurons will take a long time to fire spikes because of the large value of the threshold, which makes it hard to maintain SNN performance within a few time-steps. Second, the quantization error will be large as it is proportional to the threshold. If the conversion error is not zero for one layer, it will propagate layer by layer and will be magnified by larger quantization errors. We compare our method and the method of setting the maximum activation on the CIFAR-100 dataset. The results are reported in Table \ref{tab:thre}, where DT represents the dynamic threshold in our method. The results show that our method can achieve better performance. 

\begin{figure}[t] 
\centering
\includegraphics[width=0.40\textwidth]{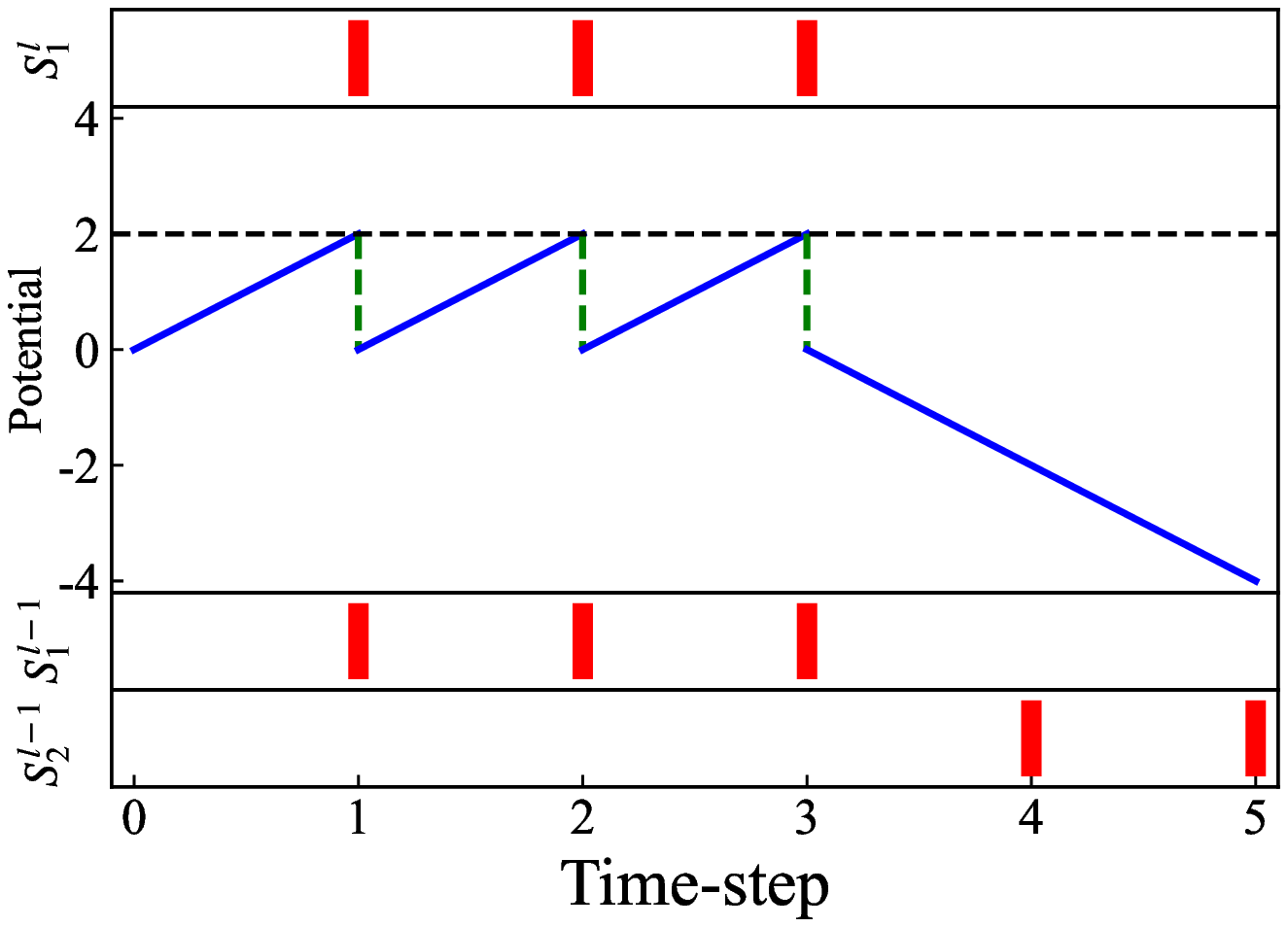} 
\caption{More spikes than expected exists for the method of setting the maximum activation.}
\label{figS1} 
\end{figure}

\textbf{Theorem 3.} 
\textit{If the threshold is set to the maximum  value of ANN activation, that is} $\theta^l = \max_{\bm{s} \in \{0,1\}^n} \left ( \max(\theta^{l-1}\bm{W}^l\bm{s}) \right)$, \textit{and} $v^l_i(0) < \theta^l$. \textit{Then at any time-step, the membrane potential of each neuron after spike} ${v}_i^l(t)$ \textit{will be less than} $\theta^l$\textit{, where $i$ represents the index of each neuron.}
\begin{proof}
 We prove it by induction. For $t=0$, it is easy to see $v^l_i(0) < \theta^l$.
 For $t>0$, we suppose that $v^l_i(t-1) < \theta^l$.
Since we have set the threshold to the maximum possible input, and $x^{l-1}_i(t)$ represents the input from layer $l-1$ to the $i$-th neuron in layer $l$, $x^{l-1}_i(t)$ will be no larger than $\theta^l$ for arbitrary $t$. 
Thus we have 
\begin{align}
     m^l_i(t) &= v^l_i(t-1) + x^{l-1}_i(t) < \theta^l + \theta^l=2 \theta^l,\\
    s^l_i(t)&=H(m^l_i(t)-\theta^l),\\
     v^l_i(t) &= m^l_i(t) - s^l_i(t)\theta^l.
\end{align}
If $\theta^l \leqslant m^l_i(t)< 2\theta^l$, then we have $v^l_i(t)=m^l_i(t) - \theta^l< \theta^l$. If $m^l_i(t) < \theta_l$, then $~v^l_i(t)=m^l_i(t) < \theta_l$.  
By mathematical induction,   $v^l_i(t)<\theta^l$ holds for any $t \geqslant  0$. 
\end{proof}

\begin{table}[t]
\caption{Comparison between our method and the method of setting the maximum activation.}
\label{tab:thre}
\centering
\scalebox{0.9}
{
\begin{threeparttable}
\begin{tabular}{llllllllll}
\toprule
DT\tnote{1}   & w/o shift  & T=4   & T=8   & T=16  & T=32  & T=64  & T=128 & T=256 & T$\geq$512 \\ \midrule[1pt]
\multicolumn{10}{c}{\textbf{VGG-16 on CIFAR-100 with L=4}}                  \\ \midrule[1pt]
$\color{red}\checkmark$ & $\color{red}\checkmark$ & 69.62\% & 73.96\% & 76.24\% & 77.01\% & 77.10\%  & 77.05\% & 77.08\% & 77.08\%  \\ \midrule
$\color{red}\checkmark$ & $\color{red}\times$ & 21.57\% & 41.13\% & 58.92\% & 65.38\% & 64.19\% & 58.60\% & 52.99\% & 49.41\% \\ \midrule
$\color{red}\times$ & $\color{red}\checkmark$ & 1.00\% & 0.96\% & 1.00\% & 1.10\% & 2.41\%  & 13.76\% & 51.70\% & 77.10\% \\ \midrule
$\color{red}\times$ & $\color{red}\times$ & 1.00\% & 1.00\% & 0.90\% & 1.00\% & 1.01\% & 2.01\% & 19.59\% & 70.86\% \\ \bottomrule
\end{tabular}
    \begin{tablenotes}
        \footnotesize
			\item[1] Dynamic threshold.
    \end{tablenotes}
\end{threeparttable}
}
\end{table}

\subsection{Effect of quantization steps L}
Table~\ref{tab_quan} reports the performance of converted SNNs with different quantization steps $L$ and different time-steps $T$. For VGG-16 and quantization steps $L=2$, we achieve an accuracy of 86.53\% on CIFAR-10 dataset and an accuracy of 61.41\% on CIFAR-100 dataset with 1 time-steps. When the quantization steps $L=1$, we cannot train the source ANN.
\begin{table}[t]
 \caption{Influence of different quantization steps.}
     \label{tab_quan}
\centering
\begin{tabular}{lllllllll}
\toprule
quantization\\steps   & T=1   & T=2   & T=4   & T=8   & T=16  & T=32  & T=64  & T=128 \\ \midrule
\multicolumn{9}{c}{\textbf{VGG-16 on CIFAR-10}}                   \\ \midrule
L=2  & 86.53\% & 91.98\% & 93.00\%  & 93.95\% & 94.18\% & 94.22\% & 94.18\% & 94.14\% \\ \midrule
L=4  & 88.41\% & 91.18\% & 93.96\% & 94.95\% & 95.40\%  & 95.54\% & 95.55\% & 95.59\% \\ \midrule
L=8  & 62.89\% & 83.93\% & 91.77\% & 94.45\% & 95.22\% & 95.56\% & 95.74\% & 95.79\% \\ \midrule
L=16 & 61.48\% & 76.76\% & 89.61\% & 93.03\% & 93.95\% & 94.24\% & 94.25\% & 94.22\% \\ \midrule
L=32 & 13.05\% & 73.33\% & 89.67\% & 94.13\% & 95.31\% & 95.66\% & 95.73\% & 95.77\% \\ \midrule
\multicolumn{9}{c}{\textbf{ResNet-20 on CIFAR-10}}                \\ \midrule
L=2  & 77.54\% & 82.12\% & 85.77\% & 88.04\% & 88.64\% & 88.79\% & 88.85\% & 88.76\% \\ \midrule
L=4  & 62.43\% & 73.2\%  & 83.75\% & 89.55\% & 91.62\% & 92.24\% & 92.35\% & 92.35\% \\ \midrule
L=8  & 46.19\% & 58.67\% & 75.70\%  & 87.79\% & 92.14\% & 93.04\% & 93.34\% & 93.24\% \\ \midrule
L=16 & 30.96\% & 39.87\% & 57.04\% & 79.5\%  & 90.87\% & 93.25\% & 93.44\% & 93.48\% \\ \midrule
L=32 & 22.15\% & 27.83\% & 43.56\% & 70.15\% & 88.81\% & 92.97\% & 93.48\% & 93.48\% \\ \midrule
\multicolumn{9}{c}{\textbf{VGG-16 on CIFAR-100}}                  \\ \midrule
L=2  & 61.41\% & 64.96\% & 68.0\%  & 70.72\% & 71.87\% & 72.28\% & 72.35\% & 72.4\%  \\ \midrule
L=4  & 57.5\%  & 63.79\% & 69.62\% & 73.96\% & 76.24\% & 77.01\% & 77.1\%  & 77.05\% \\ \midrule
L=8  & 44.98\% & 52.46\% & 62.09\% & 70.71\% & 74.83\% & 76.41\% & 76.73\% & 76.73\% \\ \midrule
L=16 & 33.12\% & 41.71\% & 53.38\% & 65.76\% & 72.80\%  & 75.6\%  & 76.37\% & 76.36\% \\ \midrule
L=32 & 15.18\% & 21.41\% & 32.21\% & 50.46\% & 67.32\% & 74.6\%  & 76.18\% & 76.24\% \\ \midrule
\multicolumn{9}{c}{\textbf{ResNet-20 on CIFAR-100}}               \\ \midrule
L=2  & 38.65\% & 47.35\% & 55.23\% & 59.69\% & 61.29\% & 61.5\%  & 61.03\% & 60.81\% \\ \midrule
L=4  & 25.62\% & 36.33\% & 51.55\% & 63.14\% & 66.70\%  & 67.47\% & 67.47\% & 67.41\% \\ \midrule
L=8  & 13.19\% & 19.96\% & 34.14\%& 55.37\% & 67.33\% & 69.82\% & 70.49\% & 70.55\% \\ \midrule
L=16 & 6.09\%  & 9.25\%  & 17.48\% & 38.22\% & 60.92\% & 68.70\%  & 70.15\% & 70.20\%  \\ \midrule
L=32 & 5.44\%  & 7.41\%  & 13.36\% & 31.66\% & 58.68\% & 68.12\% & 70.12\% & 70.27\% \\ 
\bottomrule
\end{tabular}
\end{table}

\subsection{Comparison with state-of-the-art supervised training methods on CIFAR-10 dataset}
Notably, our ultra-low latency performance is comparable with other state-of-the-art supervised training methods. Table~\ref{tab_comp_other_method} reports the results of hybrid training and backpropagation methods on CIFAR-10. The backpropagation methods require sufficient time-steps to convey discriminate information. Thus, the list methods need at least 5 time-steps to achieve $\sim$91\% accuracy.
On the contrary, our method can achieve 94.73\% accuracy with 4 time-steps. Besides, the hybrid training method requires 200 time-steps to obtain 92.02\% accuracy because of further training with STDB, whereas our method achieves 93.96\% accuracy with 4 time-steps. 
\begin{table}[t]
\caption{Compare with state-of-the-art supervised training methods on CIFAR-10 dataset}
\label{tab_comp_other_method}
\centering
\scalebox{0.9}
{
\begin{threeparttable}
\begin{tabular}{lllll}
\toprule
Model & Method   & Architecture & SNN Accuracy & Timesteps \\ \midrule
\multicolumn{5}{c}{\textbf{CIFAR-10}}                  \\ \midrule
HC & Hybrid   & VGG-16     & 92.02    & 200 \\ \midrule
STBP   & Backprop & CIFARNet   & 90.53    & 12  \\ \midrule
DT   & Backprop & CIFARNet   & 90.98    & 8   \\ \midrule
TSSL  & Backprop & CIFARNet   & 91.41    & 5   \\ \midrule
DThIR\tnote{1}  & ANN-SNN & cNet   & 77.10    & 256   \\ \midrule
\textbf{Ours}  & ANN-SNN  & VGG-16       & 93.96    & 4  \\ \midrule
\textbf{Ours}  & ANN-SNN  & CIFARNet\tnote{2}    & 94.73    & 4   \\ \bottomrule
\end{tabular}
\begin{tablenotes}
		\footnotesize
		\item[1] Implemented on Loihi neuromorphic processor
		\item[2] For CIFARNet, we use the same architecture as \cite{wu2018STBP}.
\end{tablenotes}
\end{threeparttable}
}
\end{table}

\subsection{Comparison on CIFAR-100 dataset}
Table~\ref{tab_acc_convert_cifar100} reports the results on CIFAR-100, our method also outperforms the others both in terms of high accuracy and 
ultra-low latency.  For VGG-16, the accuracy of the proposed method is 3.46\% higher than SNNC-AP and 69.37\% higher than RTS when $T=32$. When the time-steps is only 4, we can still achieve an accuracy of 69.62\%. These results demonstrate that our method outperforms the previous conversion methods. 

\begin{table}[ht]
    \caption{Comparison between the proposed method and previous works on CIFAR-100 dataset.}
    \label{tab_acc_convert_cifar100}
\centering
\renewcommand\arraystretch{1.2}
\scalebox{0.90}
{
\begin{threeparttable}
\begin{tabular}{@{}clllllllll@{}} \hline
Architecture & Method         & ANN & T=2   & T=4  & T=8  & T=16  & T=32 & T=64 &T$\geq$512 \\ \hline
\multirow{5}{*}{VGG-16}
&RMP    & 71.22\%    & -     & -     & -     & -     & -    & -   & 70.93\%               \\ \cline{2-10}
&TSC        & 71.22\%    & -     & -     & -     & -     & -    & -   & 70.97\%               \\ \cline{2-10}
&RTS       & 77.89\%    & -     & -     & -     & -    & 7.64\%   & 21.84\%   & 77.71\%               \\ \cline{2-10}
&SNNC-AP        & 77.89\%    & -     & -     & -     & -    & 73.55\%  & 76.64\%   & 77.87\%               \\ \cline{2-10}
&\textbf{Ours}  & 76.28\%    & 63.79\% & 69.62\% & 73.96\% & 76.24\% & 77.01\% &  77.10\% & 77.08\%               \\ \hline
\multirow{3}{*}{ResNet-20}
&RMP      & 68.72\%    & -     & -     & -     & -    & 27.64\%  & 46.91\%  & 67.82\%               \\ \cline{2-10}
&TSC          & 68.72\%    & -     & -     & -     & -     & -     & -      & 68.18\%               \\ \cline{2-10}
&\textbf{Ours}   & 69.94\%    & 19.96\% & 34.14\% & 55.37\% & 67.33\% & 69.82\% & 70.49\% & 70.50\%               \\ \hline
\multirow{3}{*}{ResNet-18}
&RTS    & 77.16\%    & -     & -     & -   & -   & 51.27\%   & 70.12\%   & 77.19\%               \\ \cline{2-10}
&SNNC-AP    & 77.16\%    & -     & -     & -   & -   & 76.32\%   & 77.29\%   & 77.25\%               \\ \cline{2-10}
&\textbf{Ours}        & 78.80\%    & 70.79\% & 75.67\% & 78.48\% & 79.48\% & 79.62\% & 79.54\% & 79.61\%               \\ \hline
\end{tabular}
    \begin{tablenotes}
		\footnotesize
		\item[1] RTS and SNNC-AP use altered ResNet-18, while ours use standard ResNet-18.
    \end{tablenotes}
\end{threeparttable}
}
\end{table}

\subsection{Energy consumption analysis}
We evaluate the energy consumption of our method and the compared methods~\citep{li2021free, deng2020optimal} on CIFAR-100 datasets. Here we use the same network structure of VGG-16. Following the analysis in \cite{merolla2014million}, we use synaptic operation (SOP) for SNN to represent the required basic operation numbers to classify one image. We utilize 77fJ/SOP for SNN and 12.5pJ/FLOP for ANN as the power consumption baseline, which is reported from the ROLLS neuromorphic processor \citep{qiao2015reconfigurable}. Note that we do not consider the memory access energy in our study because it depends on the hardware. As shown in Table~\ref{tab:energy}, when the time-steps is the same, the energy consumption of our method is about two times of SNNC-AP. However, to achieve the same accuracy of 73.55\%, our method requires less energy consumption.

\begin{table}[h]
    \caption{Comparison of the energy consumption with previous works}
    \label{tab:energy}
\centering
\renewcommand\arraystretch{1.4}
\scalebox{0.90}
{
\begin{threeparttable}
\begin{tabular}{@{}cllllllll@{}} \hline
Method &    & ANN & T=2   & T=4  & T=8  & T=16  & T=32 & T=64  \\ \hline
\multirow{3}{*}{RTS}
&Accuracy       & 77.89\%   & -     & -     & -     & -     & 7.64\%    & 21.84\%               \\ \cline{2-9}
&OP (GFLOP/GSOP)& 0.628     & -     & -     & -     & -     & 0.508     & 0.681           \\ \cline{2-9}
&Energy (mJ)    & 7.85      & -     & -     & -     & -     & 0.039     & 0.052        \\ \hline
\multirow{3}{*}{SNNC-AP}
&Accuracy       & 77.89\%   & -     & -     & -     & -     & 73.55\%   & 76.64\%         \\ \cline{2-9}
&OP (GFLOP/GSOP)& 0.628     & -     & -     & -     & -     & 0.857     & 1.22          \\ \cline{2-9}
&Energy (mJ)    & 7.85      & -     & -     & -     & -     & 0.660     & 0.094      \\ \hline
\multirow{3}{*}{\textbf{Ours}}
&Accuracy       & 76.28\%   & 63.79\%   & 69.62\%   & 73.96\%   & 76.24\%   & 77.01\%   & 77.10\%     \\ \cline{2-9}
&OP (GFLOP/GSOP)& 0.628     & 0.094     & 0.185     & 0.364     & 0.724     & 1.444     & 2.884       \\ \cline{2-9}
&Energy (mJ)    & 7.85      & 0.007     & 0.014     & 0.028     & 0.056     & 0.111     & 0.222       \\ \hline
\end{tabular}
\end{threeparttable}
}
\end{table}

\subsection{pseudo-code for overall conversion algorithm}
In this section, we summarize the entire conversion process in Algorithm \ref{algo:training}, including training ANNs from scratch and converting ANNs to SNNs. The QCFS in the pseudo-code represents the proposed quantization clip-floor-shift function.
\begin{algorithm}[t]
	\caption{Algorithm for ANN-SNN conversion.}
	\label{algo:training}
    \textbf{Input}: ANN model $M_{\text{ANN}}(\bm{x};\bm{W})$ with initial weight $\bm{W}$; Dataset $D$; Quantization step $L$; Initial dynamic thresholds $\bm{\lambda}$; Learning rate $\epsilon$.\\
    \textbf{Output}: $M_{\text{SNN}}(\bm{x};\bm{\hat{W}})$
	\begin{algorithmic}[1]
	    \FOR{$l = 1$ to $M_{\text{ANN}}.\text{layers}$}
	        \IF {is ReLU activation}
	            \STATE Replace $\text{ReLU}(\bm{x})$ by $\text{QCFS}(\bm{x}; L, \lambda^l)$
	        \ENDIF
	        \IF {is MaxPooling layer}
	            \STATE Replace MaxPooling layer by AvgPooling layer
	        \ENDIF
	    \ENDFOR
		\FOR{$e = 1$ to epochs}
	        \FOR{length of Dataset $D$}
	            \STATE Sample minibatch $(\bm{x}^0,\bm{y})$ from $D$
	            \FOR{$l = 1$ to $M_{\text{ANN}}.\text{layers}$}
	                \STATE $\bm{x}^{l} = \text{QCFS}(\bm{W}^l\bm{x}^{l-1}; L, \lambda^l)$
	            \ENDFOR
	            \STATE Loss = $\text{CrossEntropy}(\bm{x}^{l}, \bm{y})$
	            \FOR{$l = 1$ to $M_{\text{ANN}}.\text{layers}$}
		            \STATE $\bm{W}^l \leftarrow \bm{W}^l - \epsilon \frac{\partial Loss}{\partial \bm{W}^l}$
		            \STATE $\lambda^l \leftarrow \lambda^l - \epsilon \frac{\partial Loss}{\partial \lambda^l}$
		        \ENDFOR
	       \ENDFOR
		\ENDFOR
		\FOR{$l = 1$ to $M_{\text{ANN}}.\text{layers}$}
		\STATE $M_{\text{SNN}}.\hat{\bm{W}}^{l}\leftarrow M_{\text{ANN}}.{\bm{W}}^{l}$
	    \STATE $M_{\text{SNN}}.\theta^l\leftarrow M_{\text{ANN}}.\lambda^{l}$
	    \STATE $M_{\text{SNN}}.\bm{v}^l(0)\leftarrow M_{\text{SNN}}.\theta^l/2$
	    \ENDFOR
	    \RETURN $M_{\text{SNN}}$
	\end{algorithmic}
\end{algorithm}

\end{document}